\documentclass[10pt,journal,compsoc]{IEEEtran}

\ifCLASSOPTIONcompsoc
\usepackage[nocompress]{cite}
\else
\usepackage{cite}
\fi

\usepackage{graphicx}
\usepackage{amsmath}
\usepackage{amsthm}
\usepackage{booktabs}
\usepackage[capitalize]{cleveref}

\usepackage{algorithm}
\usepackage[noend]{algpseudocode}
\usepackage{algorithmicx,algorithm}
\usepackage{multirow}
\usepackage{amssymb}
\usepackage{amsfonts}
\usepackage{array}
\usepackage{caption}
\usepackage{subcaption}
\usepackage{makecell}
\usepackage{diagbox}
\newtheorem{theorem}{Theorem}
\newtheorem{lemma}{Lemma}

\def\Mat#1{{\boldsymbol{#1}}}
\def\Vec#1{{\boldsymbol{#1}}}
\newtheorem*{definition}{Definition}

\DeclareRobustCommand\onedot{\futurelet\@let@token\@onedot}
\def\eg{\emph{e.g.}} 
\def\ie{\emph{i.e.}}

\def\etal{\emph{et al.}}

\usepackage{xurl}
\ifCLASSINFOpdf
\else
\fi
\begin{document}
	%
	\title{Toward Enhanced Robustness in Unsupervised Graph Representation Learning: A Graph Information Bottleneck Perspective}
	%
	%
	%
	%
	
	\author{Jihong~Wang, Minnan~Luo$^{*}$, Jundong~Li, Ziqi~Liu, Jun~Zhou and~Qinghua~Zheng
		\IEEEcompsocitemizethanks{\IEEEcompsocthanksitem Jihong Wang, Minnan Luo, and Qinghua Zheng are with the Ministry of Education Key Lab for Intelligent Networks and Network Security, School of Computer Science and Technology, Xi'an Jiaotong University, Xi'an 710049, China. \protect\\
			E-mail: wang1946456505@stu.xjtu.edu.cn, minnluo@xjtu.edu.cn, qhzheng@xjtu.edu.cn 
			\IEEEcompsocthanksitem Jundong Li is with the Department of Electrical and Computer Engineering, Department of Computer Science, and School of Data Science, University of Virginia, USA. \protect\\
			E-mail:  jundong@virginia.edu
			\IEEEcompsocthanksitem Ziqi Liu and Jun Zhou are with the Ant Financial Services Group, Hangzhou, Zhejiang 310000, China \protect\\
			E-mail:  ziqiliu@antgroup.com, jun.zhoujun@antgroup.com
			
			\IEEEcompsocthanksitem  Corresponding author: Minnan Luo.}
	}
	
	%
	%

	\markboth{Journal of \LaTeX\ Class Files,~Vol.~14, No.~8, August~2015}%
	{Shell \MakeLowercase{\textit{et al.}}: Bare Advanced Demo of IEEEtran.cls for IEEE Computer Society Journals}
	%



	\IEEEtitleabstractindextext{%
		\begin{abstract}
		Recent studies have revealed that GNNs are vulnerable to adversarial attacks. Most existing robust graph learning methods measure model robustness based on  label information, rendering them infeasible when label information is not available. A straightforward direction is to employ the widely used Infomax technique from typical Unsupervised Graph Representation Learning (UGRL) to learn robust unsupervised representations. Nonetheless, directly transplanting the Infomax technique from typical UGRL to robust UGRL may involve a biased assumption. In light of the limitation of Infomax, we propose a novel unbiased robust UGRL method called \emph{Robust Graph Information Bottleneck} (RGIB), which is grounded in the Information Bottleneck (IB) principle. Our RGIB attempts to learn robust node representations against adversarial perturbations by preserving the original information in the benign graph while eliminating the adversarial information in the adversarial graph. There are mainly two challenges to optimize RGIB: 1) high complexity of adversarial attack to perturb node features and graph structure jointly in the training procedure; 2) mutual information estimation upon adversarially attacked graphs. To tackle these problems, we further propose an efficient adversarial training strategy with only feature perturbations and an effective mutual information estimator with subgraph-level summary. Moreover, we theoretically establish a connection between our proposed RGIB and the robustness of downstream classifiers, revealing that RGIB can provide a lower bound on the adversarial risk of downstream classifiers. Extensive experiments over several benchmarks and downstream tasks demonstrate the effectiveness and superiority of our proposed method.
		
		\end{abstract}

		}

	\maketitle

	\IEEEdisplaynontitleabstractindextext

	%
	\IEEEpeerreviewmaketitle

	\ifCLASSOPTIONcompsoc
	\IEEEraisesectionheading{\section{Introduction}\label{sec:introduction}}
	\else
	\section{introduction}
	\label{sec:introduction}
	\fi

	%
	%
	%
	%
	\IEEEPARstart{U}{nsupervised} graph representation learning (UGRL)~\cite{velickovic2019deep,peng2020graph,8392745,8519335,8941296} recently gained significant attention from researchers, due to its potential applications in various domains, such as chemistry, social media and bioinformatics~\cite{ye2020symmetrical,fung2021benchmarking,8326519,9001178,yuan2020xgnn}, among others. UGRL aims to embed the nodes of a graph into a low-dimensional space and extract the most meaningful information for downstream tasks without relying on label information. Compared with supervised methods~\cite{kipf2016semi,hamilton2017inductive,velivckovicgraph}, UGRL achieves competitive performance while avoiding the need for expensive manual labels.
 
    Graph Neural Networks (GNNs)~\cite{kipf2016semi,hamilton2017inductive,8901123,feng2019graph} have recently emerged as a promising learning paradigm in UGRL owing to their superior expressive capabilities. Despite their success, robustness of GNNs remains a critical challenge in UGRL. Numerous recent studies~\cite{zugner2018adversarial,zugner2019adversarial,xu2019topology} have shown that GNNs are vulnerable to adversarial attacks, \ie, slight perturbations on graph structure or node features may mislead the models into making incorrect predictions. This vulnerability could lead to severe consequences in certain applications. For example, GNNs are widely used in the recommendation systems of various e-commerce platforms~\cite{fan2020graph,ying2018graph,fan2019graph}. On these platforms, malicious attackers may employ vicious accounts to interact synergistically with both popular and unpopular items. By exploiting vulnerabilities in GNN-based recommendation systems, attackers could deceive these recommendation systems into increasing the visibility of low-quality, lesser-known products. Consequently, such malicious behaviors may undermine user experience and compromise the overall effectiveness of recommendation systems.

 
 

\begin{figure}[t]
    \centering
    \includegraphics[width=0.4\textwidth]{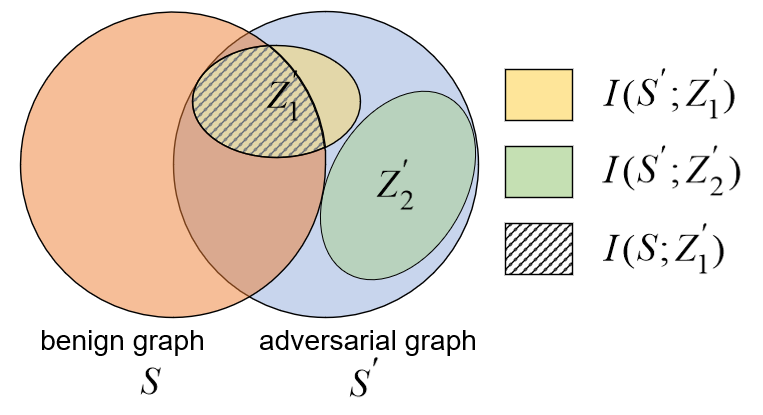}
    \caption{An intuitive view of our motivation. Graph domain is represented by a circle, and the size of the circle indicates the entropy of graph. $S$ and $S^\prime$ denotes the benign graph and adversarial graph, respectively. Two ellipses indicate two different adversarial representations $Z^\prime_1$ and $Z^\prime_2$.}
    \label{fig:motivation}
\end{figure}

	\begin{figure*}[ht]
	\setlength{\belowcaptionskip}{-0.5cm}
        \centering 
        \includegraphics[width=0.9\textwidth]{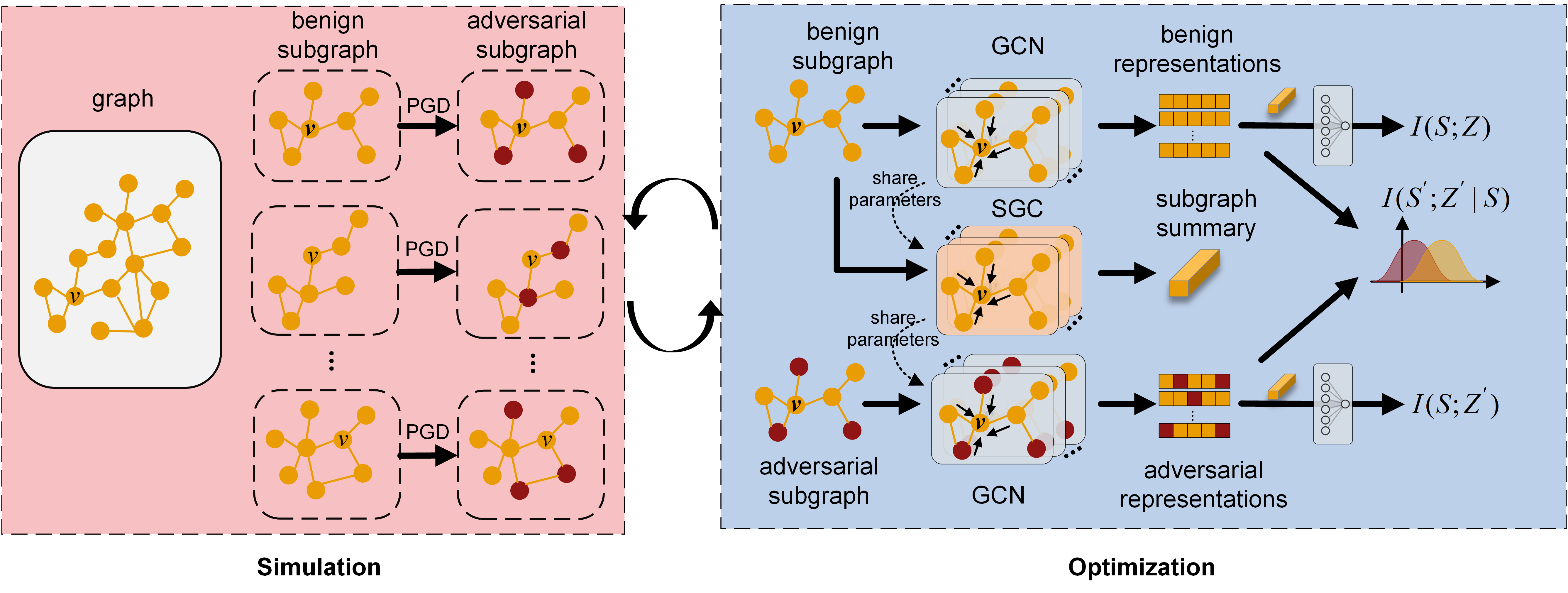} 
        \caption{An overview of our proposed RGIB method which consists of two primary steps. In the simulation step, we conduct PGD attack to generate feature perturbations efficiently. In the optimization step, we effectively optimize the mutual information to retain the original information in the benign graph while discarding the adversarial information in the adversarial graph.} 
        \label{fig:intro} 
        \end{figure*}

    To address the security and privacy concerns arising from the vulnerability of GNNs, recent studies~\cite{jin2019power,zhang2020gnnguard,xu2019topology} have proposed several defense methods against graph adversarial attacks, which fall under the umbrella of \emph{robust graph representation learning}. However, the majority of these studies necessitate access to label information, rendering them inapplicable to UGRL. In this paper, we delve into the robust UGRL problem. Specifically, given an input graph, our objective is to learn a robust encoder that maps nodes to low-dimensional representations, which remain robust even when certain edges or node features are perturbed adversarially. To address the robust UGRL problem, a straightforward direction is to employ the widely used Infomax~\cite{velickovic2019deep,peng2020graph} technique from typical UGRL to learn robust representations. For example, GRV~\cite{xu2022unsupervised}, a robust UGRL method, achieves state-of-the-art robust performance by maximizing the mutual information (MI) between the adversarial graph and its corresponding adversarial representations. Nonetheless, we argue that directly transplanting the Infomax technique from the typical UGRL to the robust UGRL problem involves an assumption, which posits that an encoder is robust if the MI remains stable before and after an attack~\cite{xu2022unsupervised}. Upon an in-depth theoretical analysis of the Infomax objective, we discover that the involved assumption is biased, since the encoder might achieve high MI even if it solely embeds the adversarial information. We provide an intuitive explanation for our theoretical analysis in \cref{fig:motivation}. In the example, we represent the graph domain with a circle and the representation domain with an ellipse. Suppose there are two graphs: a benign graph $S$ and an adversarial graph $S^\prime$. Two representations, $Z^\prime_1$ and $Z^\prime_2$, are learned based on the adversarial graph. The mutual information between the representations and graphs can be abstracted as their overlapping area (\eg, $I(S^\prime, Z_1^\prime)$). As illustrated in \cref{fig:motivation}, although $Z^\prime_2$ achieves higher mutual information with $S^\prime$ than $Z^\prime_1$, \ie, $I(S^\prime;Z_2^\prime) > I(S^\prime;Z_1^\prime)$, $Z_2^\prime$ is undesirable since it only embeds adversarial information, without incorporating any information from the benign graph. Theoretical analysis in our paper further proves this intuition. 
 
    In light of the limitations associated with directly employing Infomax technique in robust UGRL, we propose a novel and unbiased approach to robust UGRL that is grounded in the Information Bottleneck (IB) principle~\cite{tishby2000information,alemi2016deep}. The IB principle aims to learn robust representations by eliminating irrelevant input information and retaining only the information relevant to a particular task, thereby reducing the impact of extraneous factors or noise nuisance. Similarly, we propose to discard the adversarial information in the adversarial graph while preserving the original information in the benign graph. The IB principle endows our method freedom from the biased assumption involved in the Infomax technique. Specifically, given the adversarial graph $S^{\prime}$ and the representation $Z_1^\prime$ in \cref{fig:motivation}, we decompose the mutual information $I(S^\prime;Z_1^\prime)$ into two complementary terms: the informative term (the yellow shadow part in \cref{fig:motivation}) and the adversarial term (the yellow part without shadow). The informative term preserves original information associated with the benign graph $S$, whereas the adversarial term encompasses the adversarial information in $S^\prime$ that may impede the performance of representations. To learn robust representations, our method, named Robust Graph Information Bottleneck (RGIB), aims to discard adversarial information and preserve original information by concurrently maximizing the informative term and minimizing the adversarial term. Additionally, we theoretically prove that our RGIB method benefits the robustness of downstream classifications, regardless of the specific forms of the classifiers. The proof demonstrates the validity and effectiveness of our RGIB theoretically.

    The implementation of RGIB method consists of two primary steps: the simulation step and the optimization step. In the simulation step, we consider a scenario where the prior knowledge about the adversarial attacker is inaccessible. Instead, we propose to simulate the attackers' behavior by considering the worst-case scenario, where adversarial subgraphs are generated to optimize our RGIB objective in the least favorable direction. In the optimization step, we simultaneously maximize the informative term and minimize the adversarial term to learn robust representations. There are two primary challenges confronting us in the two steps, correspondingly.
    The first challenge centers around the generation of adversarial graphs. Adversarial perturbations can be introduced into both the structure and node features of a graph. Nonetheless, perturbing the structure of a graph is a combinatorial problem and can lead to prohibitively expensive computational cost~\cite{zugner2019adversarial,xu2019topology} in the simulation step. To overcome this challenge, we propose an efficient adversarial training strategy with only feature perturbations based on the intuition that feature perturbations and structure perturbations behave somewhat similarly in GNNs to some extent. Furthermore, the intuition is supported by a theoretical proof of the Simplifying Graph Convolutional Networks (SGC)~\cite{wu2019simplifying} and empirical studies on other GNNs. The second challenge involves optimizing mutual information on adversarial graphs. Although some methods~\cite{velickovic2019deep,peng2020graph} have been proposed to estimate mutual information on graphs in a contrastive fashion, they all assume that the graphs are free from adversarial perturbations. Given that the adversarial perturbations can impede the distinguishing of positive and negative samples in the MI estimation, we propose a fine-grained mutual information estimator named \emph{Subgraph Mutual Information} (SMI) upon the adversarial graph. Our SMI provides fine-grained subgraph-level summaries to estimate mutual information, which enables representations to eliminate the influence of adversarial perturbations more effectively. Experimental results demonstrate that our proposed SMI works well on adversarial graphs. In summary, our main contributions in this paper are summarized as follows:
    \begin{itemize}
    \item We propose a novel robust UGRL method named \textit{Robust Graph Information Bottleneck} (RGIB) based on the IB principle. Our method can preserve the original information in the benign graph while discarding the adversarial information induced by adversarial perturbations. We also theoretically demonstrate that our RGIB benefits downstream classification robustness regardless of the specific form of the classifier. 
    \item We propose an efficient feature-only adversarial training strategy for the optimization of RGIB as well as an effective mutual information estimator with a fine-grained subgraph-level summary upon the adversarial graph.
    \item Extensive experimental results across various downstream tasks, including classification, anomaly detection, and link prediction, demonstrate the robustness of our method in learning node representations that are resilient to adversarial attacks.
    \end{itemize}

	\section{related work}\label{sec:related}
	In this section, we briefly review related works on two fields closely relevant to our study, \ie,  unsupervised graph representation learning and robust graph representation learning.
	\subsection{Unsupervised Graph Representation Learning} 
	Unsupervised graph representation learning aims to learn low-dimensional representations for nodes in a graph. Traditional methods are usually based on random walk   \cite{perozzi2014deepwalk,tang2015line} or matrix factorization techniques \cite{qiu2018network,yang2008non}. With the rapid development of graph neural networks (GNNs), numerous unsupervised graph representation learning algorithms \cite{kipf2016variational,garcia2017learning,hamilton2017inductive,velickovic2019deep} based on GNNs are proposed. These methods try to learn representations in an autoencoder manner, \ie, the graph is embeded to low-dimensional node representations by a GNN-based encoder, and then the node representations are optimized by the reconstruction error. For example, Pan \etal \cite{pan2018adversarially} proposed to train a GNN-based auto-encoder in an adversarial style to force the representations to follow the Gaussian distribution; GraphSage \cite{hamilton2017inductive} learns an unsupervised GNN encoder through a random-walk based objective in an inductive way. However, these auto-encoder based methods overemphasize proximity information \cite{velickovic2019deep} and suffer from unstructured predictions \cite{tian2020contrastive}. 
	
	To overcome the shortcomings mentioned above, some works \cite{li2019graph,velickovic2019deep,sun2019infograph,zhu2021graph,hassani2020contrastive} try to learn representations by contrastive learning instead of directly optimizing the reconstruction error. Theoretically, these contrastive learning methods \cite{velickovic2019deep,zhu2021graph} maximize the mutual information instead of overemphasizing proximity information. For example, Zhu \etal \cite{zhu2021graph} proposed to learn node representations with adaptive augmentation that incorporates various priors for topological and semantic aspects of the graph in a contrastive manner. Velickovic \etal~\cite{velickovic2019deep} proposed to learn node representations through contrasting node and graph encodings and achieve decent performance on several benchmarks for node classification. 
	Despite their effectiveness, existing unsupervised graph representation learning methods mainly focus on learning effective node representations, while the robustness of the learned representations against adversarial attack is often overlooked.
	
	\subsection{Robust Graph Representation Learning}
	As recent studies \cite{dai2018adversariala,zugner2018adversarial,zugner2019adversarial,bojchevski2019adversarial} reveal that existing graph neural networks are vulnerable to adversarial attacks, robust graph representation learning attracts a surge of interests in recent years. Some defense methods \cite{jin2019power,jin2019power, xu2019topology,wang2019adversarial,jin2020graph,wu2019adversarial,entezari2020all} are proposed to eliminate the vicious influence of adversarial examples. According to the used defense strategy, existing defense methods can be divided into three categories, including model-based \cite{jin2019power,jin2019power}, training-based \cite{xu2019topology,wang2019adversarial,jin2020graph}, and preprocessing-based \cite{wu2019adversarial,entezari2020all} methods. 
	For model-based methods, Ming \etal\cite{jin2019power} designed an alternative operator based on graph powering to replace the classical Laplacian in GNN models. They demonstrate that the combination of this operator with vanilla GCN can help defense against evasion attacks. Zhang \etal. \cite{zhang2020gnnguard} adopted neighbor importance estimation and layer-wise graph memory components to increase the robustness against various attacks. For training-based methods, Xu \etal\cite{xu2019topology} proposed a topology attack method based on projected gradient descent, and the attack method is used to improve the robustness of GNNs by adversarial training. Wang \etal\cite{wang2019adversarial} leveraged adversarial contrastive learning to improve the robustness of GNN models. The proposed method applies conditional GAN to utilize graph-level auxiliary information. Jin \etal\cite{jin2020graph}  presented Pro-GNN, which jointly learns clean graph structure and trains robust GNN models together. For preprocessing-based methods, Jin \etal\cite{wu2019adversarial} proposed to compute the Jaccard Similarity to remove suspicious edges between suspicious nodes. Entezari \etal\cite{entezari2020all} found that graph adversarial attack tends to generate graphs with high-rank adjacency matrixs. Thus the researchers propose to reduce the effect of attacks by computing the low-rank approximation of the graphs before training GNN models.
	
	Note that all of the previous defense methods against graph adversarial attacks focus on semi-supervised learning, which may be impractical in some situations when labels are rare and expensive. In this paper, we focus on robust graph representation learning without any label information. The most similar work to ours is the work proposed in \cite{xu2022unsupervised}. The authors propose a robustness measure named \textit{graph representation vulnerability} (GRV), which does not rely on any supervision signals and achieves competitive learning performance on the adversarially perturbed graph. However, just as mentioned above, GRV employs a biased assumption and may involve adversarial information in the learned representations.


	\section{notation and preliminary}
	In this section, we first elaborate on the notions used in this paper.Then, we briefly introduce the Information Bottleneck principle which our method is built upon. 
    \subsection{Notation}
    In this paper, we use upper-case letters (\eg, $X$ and $Y$) to denote random variables, the corresponding calligraphic letters (\eg, $\mathcal{X}$ and $\mathcal{Y}$) to denote their support, and the corresponding lower-case (\eg, $x$ and $y$) to denote the realizations of the random variables. $p(X)$ denotes the probability distribution of the corresponding random variables. $p(X=x)$ indicates the probability density and is simplified as $p(x)$ for convenience. We denote $\mathcal{P}(\mathcal{X})$ the probability measures on $\mathcal{X}$. Let $(\mathcal{X}, \Delta)$ be a metric space, where $\Delta: \mathcal{X} \times \mathcal{X} \rightarrow \mathbb{R}$ is a distance metric. $\mathcal{B}_\Delta(x, \epsilon) = \{x^\prime \in \mathcal{X}:\Delta(x^\prime, x) \leq \epsilon\}$ is the ball around $x$ with radius $\epsilon$. Moreover, the bold upper-case letters are used to denote matrices (\eg, $\Mat{A}$ and $\Mat{X}$) and bold lower-case letters denote vectors (\eg, $\Vec{a}$ and $\Vec{x}$). 
	
	In the context of unsupervised graph representation learning, we define an attributed graph $\mathcal{G}=\left( \mathcal{V}, \mathcal{E}, \mathcal{X}\right)$, where $\mathcal{V}=\{v_1,v_2,\cdots, v_{|V|}\}$ represents the node set; $\mathcal{E}=\{e_1, e_2,\cdots, e_{|\mathcal{E}|}\} \subseteq \mathcal{V} \times \mathcal{V}$ refers to the edge set and $\mathcal{X} = \{ \Vec{x_1}, \Vec{x_2},\cdots, \Vec{x_{|\mathcal{V}|}}\}$ collects the feature vectors of all nodes. Since node representations are learned in a message-aggregation fashion, they can be determined by the corresponding receptive field in the graph. The receptive field of a node can be denoted as the $k-$hop subgraph $S$ around the node where $k$ corresponds to the layer number of the GNN encoder adopted. For convenience of analysis, we ignore the irrelevant information in the graph $G$ and focus on the receptive field $S$. A realization of $S$ can be formulated as $s_i = (\Mat{A}_i,\Mat{X}_i)$, where $\Mat{A}_i,\Mat{X}_i$ denote the adjacency matrix and feature matrix corresponding to a specific node $i$ and its $k-$hop neighbors. The goal of node-level unsupervised representation learning is to learn an encoder that maps each subgraph to a representation vector, formally: $e: \mathcal{S}\rightarrow \mathcal{Z}$ where $\mathcal{S}$ denotes the support of subgraphs $S$ and $\mathcal{Z}$ corresponds to the support of representations $Z$. We denote the adversarial subgraph and its corresponding representation as $S^\prime$ and $Z^\prime$. We further define $f:\mathcal{Z}\rightarrow \mathcal{Y}$ as a downstream learning task (\eg, classification) that maps a representation $\Vec{z} \in \mathcal{Z}$ to a class label, \ie, $f(z)\in \mathcal{Y}$. In this sense, $f \circ e$ denotes the composition of $f$ and $e$ such that $(f \circ e)(S) = f(e(S))$.

    \subsection{Information Bottleneck Principle}
    Given the input data $X$ and the corresponding label $Y$, Information Bottleneck~\cite{tishby2000information,alemi2016deep,federici2020} aims to discover a compressed latent representation $Z$ that is maximally informative in terms of $Y$ while discarding the irrelevant information in $X$. Formally, suppose a Markov Chain $Z-X-Y$, \ie, $Z$ is conditionally independent of $Y$ given 
    $X$, one can learn the latent representation $Z$ by optimizing the following optimization problem
    \begin{equation}
    \begin{aligned}
                \max_{Z} \mathcal{L}_{IB} &= I(Z;Y) - \beta I(X;Z) \\
                &=(1-\beta)I(Z;Y) - \beta I(Z;X|Y)
    \end{aligned}
    \end{equation}
    where $\beta$ denotes a hyper-parameter trading-off the informativeness term $I(Z;Y)$ and compression term $I(X;Z)$. In the above formulation, Mutual information (MI) measures the relevance of two random variables, and the MI between random variables $X$ and $Z$ is formulated as $I(X;Z) = \int_x\int_z p(x,z)\log \frac{p(x,z)}{p(x)p(z)}dxdz.$ 

    \section{methodology}
	In this section, we first present our novel and robust method for unsupervised graph representation learning (UGRL), called Robust Graph Information Bottleneck (RGIB) based on the information bottleneck principle. Then, to effectively and efficiently optimize the RGIB objective, we introduce a mutual information optimization and training strategy.

\subsection{Robust Graph Information Bottleneck}
    Robust UGRL aims to learn robust node representations of a graph against adversarial attacks without supervision signals. Unlike semi-supervised learning on graphs, quantifying the robustness of node representations without label information is often an intractable problem. One feasible method is to employ the mutual information as a measure to maximize the robustness of the representations, \ie, the Infomax technique~\cite{peng2020graph,velickovic2019deep}. Specifically, given a graph $S$ and its corresponding adversarial graph $S^\prime$, one can learn robust node representations $Z^\prime$ for $S^\prime$ by maximizing the mutual information $I(S^\prime;Z^\prime)$~\cite{xu2022unsupervised}. Nonetheless, as we mentioned above, directly transplanting the Infomax technique from the typical UGRL to the robust UGRL problem may involve a biased assumption, which assumes that an encoder is robust if the MI remains stable before and after an attack. The assumption is biased since the encoder might achieve high MI even if it solely embeds the adversarial information. To provide a theoretical view of the biased assumption, we formally decompose the mutual information $I(S^\prime;Z^\prime)$ to two complementary terms:
    \begin{theorem}\label{th:factor}
    Given a graph $S$ and its corresponding adversarial graph $S^\prime$. Let $Z^\prime$ be the representations on of $S^\prime$. Suppose a Markov Chain $S-S^\prime-Z^\prime$, \ie, $Z^\prime$ is conditionally independent of $S$ when $S^\prime$ is given, the mutual information between the adversarial graph and its corresponding adversarial representations can be decomposed to two complementary terms:
    \begin{equation}\label{eq:decomposition}
        I(S^\prime; Z^\prime) = I(S^\prime; Z^\prime|S) + I(S;Z^\prime)
    \end{equation}
    \begin{proof}
    According to the chain rule of mutual information, we can decompose the mutual information between $S^\prime$ and $Z^\prime$ as follows.
    \begin{equation}
    \begin{aligned} \label{eq:chain_rule}
        I(S^\prime; Z^\prime) &= I(S^\prime;Z^\prime|S) + I(S^\prime; Z^\prime; S) \\
        &= I(S^\prime; Z^\prime|S) + I(S;Z^\prime) - I(S;Z^\prime|S^\prime)
    \end{aligned}
    \end{equation}
    Since $Z^\prime$ is conditionally independent of $S$ given $S^\prime$, we have $I(S; Z^\prime|S^\prime) = 0$. Substituting the equation into \cref{eq:chain_rule}, we can complete the proof.
    \end{proof}
    \end{theorem}
    In \cref{eq:decomposition}, $I(S^\prime; Z^\prime|S)$ is called the \textit{adversarial term} since it involves adversarial information not contained in the benign graph $S$. $I(S;Z^\prime)$ is called the \textit{informative term} since it preserves the original information of $S$. Directly maximizing $I(S^\prime; Z^\prime)$ may involve undesirable information since the adversarial term is also maximized, biasing the representations to embed the adversarial information. To avoid the biased assumption, we extend the IB principle to the robust UGRL problem:
    \begin{definition} \label{definition}
    Given a graph $S$ and its corresponding adversarial graph $S^\prime$. Let $Z^\prime$ be the representations on of $S^\prime$. Suppose a Markov Chain $S-S^\prime-Z^\prime$, Robust Graph Information Bottleneck (RGIB) seeks for the robust node representations by discarding the adversarial information while preserving the original information:
    \begin{equation}
        \max_{Z^\prime} \mathcal{L}_{RGIB} = I(S;Z^\prime) - \beta I(S^\prime;Z^\prime|S)
    \end{equation}
    where $\beta > 0$ is a parameter that makes a tradeoff between two complementary terms.
    \end{definition}
    
    The RGIB shares a similar ideology to the IB principle, which retains the original information in the benign graph while discarding the adversarial information in the adversarial graph. Specifically, RGIB avoids the biased assumption by maximizing the informative term and minimizing the adversarial term, simultaneously. 
    
    Moreover, existing studies show that there is usually a trade-off between accuracy and robustness \cite{zhang2019theoretically}, which indicates that solely considering the robustness of the model may compromise its effectiveness. To ensure both the robustness and effectiveness of our RGIB, we introduce an additional term, $I(S;Z)$, which serves to guarantee the effectiveness of representations on the benign graph:
	 \begin{equation}\label{eq:opt_complete}
	 \begin{aligned}
     \max_{Z^\prime, Z} \mathcal{L}^\prime_{RGIB} = \alpha I(S;Z^\prime) + (1-\alpha) I(S;Z) - \beta I(S^\prime;Z^\prime|S)
	 \end{aligned}
	 \end{equation}
	 where $\alpha > 0$ is a trade-off hyper-parameter. By incorporating $I(S;Z)$, we strike a balance between robustness and effectiveness, thus enabling our RGIB approach to generate reliable and informative representations. The representations can withstand adversarial perturbations while maintaining high performance in real-world applications.

    As illustrated in \cref{fig:intro}, the RGIB method consists of two primary steps: the simulation step and the optimization step. In the simulation step, the adversarial graph $S^\prime$ is inaccessible in practice due to the absence of prior knowledge about the attackers. We propose to simulate the behavior of attackers by considering the worst case where the adversarial graph $S^\prime$ minimizes the objective $\mathcal{L}^\prime_{RGIB} $. Specifically, we propose to maximize the infimum of the objective $\mathcal{L}^\prime_{RGIB}$:
    \begin{equation}\label{eq:opt}
	 \begin{aligned}
     &\max_{Z^\prime, Z} \hat{\mathcal{L}}_{RGIB}\\  s.t. \  \hat{\mathcal{L}}_{RGIB} = &\inf_{p(S^\prime)\in \mathcal{B}_{W_\infty}(p(S), \epsilon)} \mathcal{L}^\prime_{RGIB} 
	 \end{aligned}
	 \end{equation}
    where $\mathcal{B}_{W_\infty}(p(S), \epsilon)$ denotes the $\infty$-Wasserstein ball~\cite{champion2008wasserstein}. The $\infty$-Wasserstein ball imposes a constraint on the extent of the adversarial perturbations applied to the graph by limiting them within a budget $\epsilon$. In the optimization step, we maximize the objective $\hat{\mathcal{L}}_{RGIB}$ to learn robust representations.
    
    The challenge of learning robust representations via RGIB is two-fold: (i) High complexity to generate adversarial perturbations on graphs, \ie, the infimum over $S^\prime$ in \cref{eq:opt}.  (ii) Mutual information optimization on adversarially perturbed graphs. To address these issues, we propose an efficient feature-only adversarial training strategy and derive tractable bounds for mutual information estimation upon the adversarial graph. We elaborate on the details of the training strategy and mutual information bound in the following.

  \subsection{Perturbations Generation}
    The problem in \cref{eq:opt} is a bi-level optimization problem. In the inner minimization problem, we attempt to obtain perturbed subgraphs that minimize the objective $\mathcal{L}^\prime_{RGIB}$, which is also referred to as the graph adversarial attack problem. Unlike other domains, such as images, adversarial attack on graphs can be achieved by manipulating both graph structure and node features. Compared to node features, the structure of a graph is highly discrete, leading to a combinatorial optimization problem. To tackle this issue, existing works \cite{xu2019topology,zugner2019adversarial} typically employ the gradients of the adjacency matrix to optimize the perturbations. However, these methods can still be computationally expensive for large graphs, rendering it impractical to perturb the structure during the training procedure. Considering that the GNNs usually extract information through message-aggregation mechanisms, perturbations on structure and features may exhibit similar behavior: perturbations on structure alter the path of message propagation, while perturbations on features change the message directly. For instance, it is evident to achieve the following \cref{th1} for the Simplifying Graph Convolutional Networks (SGC)~\cite{wu2019simplifying}.
	 \begin{theorem}\label{th1}
	 Given a Simplifying Graph Convolutional Network (SGC) $g(\Mat{A};\Mat{X};\Mat{\Theta}) = \Mat{\hat{A}}^K\Mat{X}\Mat{\Theta}$, where $\Mat{\hat{A}} = \Mat{\tilde{D}}^{-\frac{1}{2}}\Mat{\tilde{A}}\Mat{\tilde{D}}^{-\frac{1}{2}}$ is the symmetric normalized adjacency matrix built upon the adjacency matrix with self-loop $\Mat{\tilde{A}} = \Mat{A} + \Mat{I}$ and its degree matrix $\Mat{\tilde{D}}$; $K$ refers to the number of layers of SGC, and $\Mat{\Theta}$ denotes the parameter matrix. 
	 For any perturbations $\Delta \Mat{A}$ on the adjacency matrix, there must be corresponding perturbations $\Delta \Mat{X}$ on features, such that
	 \begin{equation}\label{eq:sgc}
	 g(\Mat{A}+\Delta \Mat{A};\Mat{X};\Mat{\Theta}) = g(\Mat{A};\Mat{X}+\Delta \Mat{X};\Mat{\Theta}).
	 \end{equation}
	 \end{theorem}
	 \begin{proof}
	 	The theorem can be easily proved by substituting the following equation into \cref{eq:sgc}:
	 	 \begin{equation}
	 	\Delta \Mat{X} = \left(\Mat{\hat{A}}^{+}\right)^K\left(\Mat{\hat{A}}^\prime\right)^K \Mat{X} - \Mat{X},
	 	\end{equation} 
	 	where $\Mat{\hat{A}}^\prime$ is the symmetric normalized adjacency matrix of the perturbed adjacency matrix $\Mat{A}^\prime = \Mat{A}+\Delta \Mat{A}$; $\Mat{\hat{A}}^{+}$ denotes the pseudo-inverse of $\Mat{\hat{A}}$. The proof is completed.
	 \end{proof}
    
    \cref{th1} suggests that given any structure perturbations, we can induce feature perturbations that exert the same influence on SGC. In other words, in the context of SGC, feature perturbations behave similarly to structure perturbation. To further validate whether this similarity prevails across general GNNs, we conduct empirical studies on the Cora~\cite{sen2008collective} dataset with various GNNs, including GCN~\cite{kipf2016semi}, Graphsage~\cite{hamilton2017inductive}, SGC~\cite{wu2019simplifying} and GAT~\cite{velivckovicgraph}. 
    Specifically, given a GNN denoted as $g(\cdot)$, we introduce random perturbations $\Delta \Mat{A}$ to the graph structure and seek feature perturbations $\Delta \Mat{X}$ that emulate the impact of the structure perturbations on the GNN. The similarity between the effects that feature perturbations and structure perturbations on the GNN can be measured by the Mean Square Error (MSE), \ie, $\text{MSE}\left(g\left(\Mat{A}+\Delta \Mat{A};\Mat{X}\right)-g\left(\Mat{A};\Mat{X}+\Delta \Mat{X}\right)\right)$. We identify feature perturbations analogous to structural perturbations by minimizing the MSE using the Stochastic Gradient Descent (SGD) optimizer~\cite{lecun2002efficient}, with the initial value set as $\Delta \Mat{X}=0$. We present the MSE before and after optimization in \cref{fig:equivalency}. For instance, in the case of GCN, the MSE is $0.30$ when there is no feature perturbations (\ie, $\Delta \Mat{X} = \bf{0}$) and it converges to $0.01$ with optimized feature perturbations. The empirical study suggests that feature perturbations can behave similarly to structure perturbations on general GNNs.

   \begin{figure}[t]
    \centering
    \includegraphics[width=0.4\textwidth]{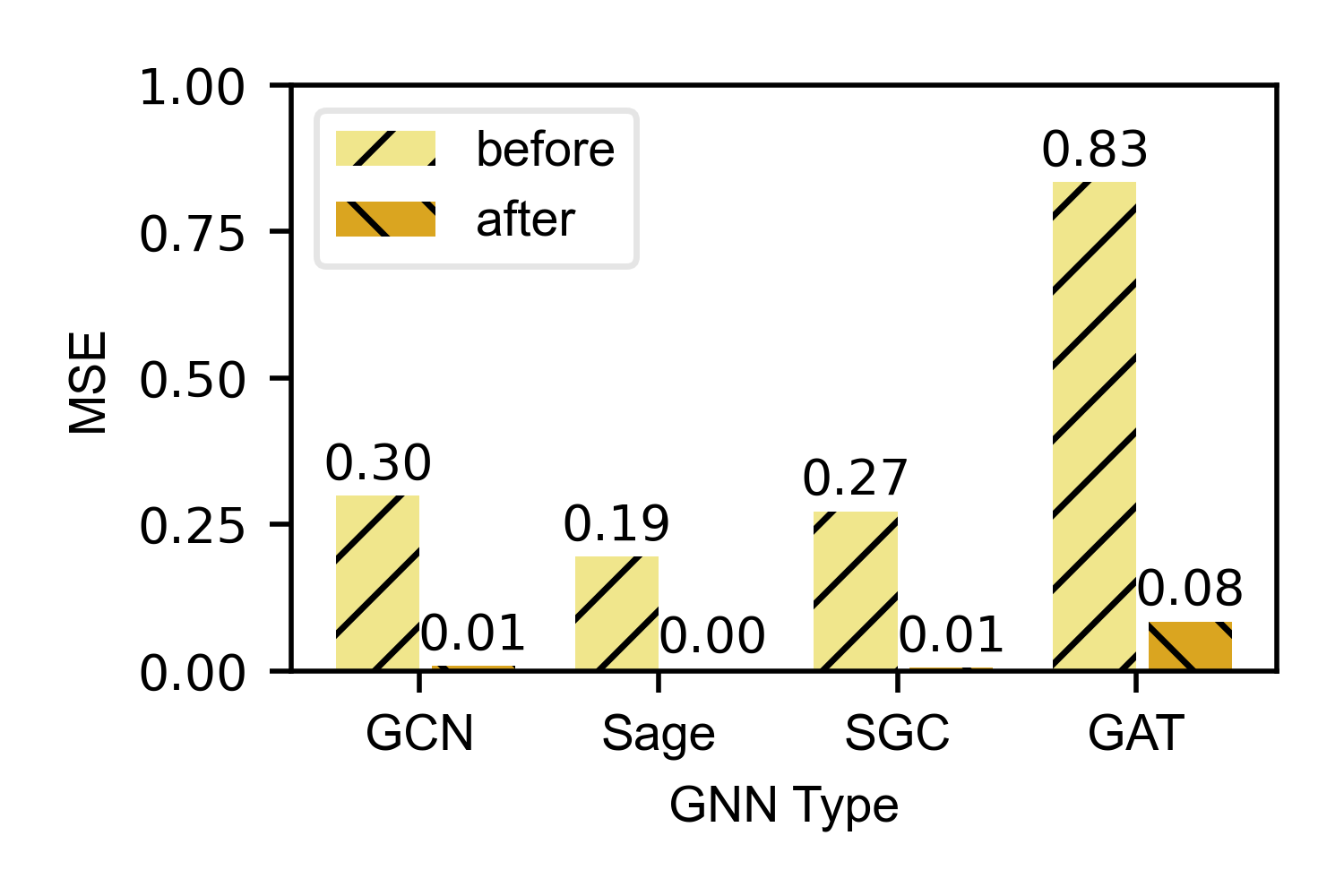}
    \caption{Empirical study on the similarity of graph perturbations. We first randomly perturb the structure of the graph and then simulate the influence of structure perturbations by perturbing the features in the original graph. The yellow bar indicates the MSE before optimization, where there is no feature perturbations. The orange bar indicates the MSE after optimization, where we simulate the structure perturbations with feature perturbations.}
    \label{fig:equivalency}
\end{figure}

   The theoretical and empirical studies support the notion that a feature-only adversarial training strategy is sufficient for modeling the effects of structural perturbations, enabling efficient learning of robust representations in GNNs without direct manipulation of graph structure. To generate perturbations on features, we adopt the Projected Gradient Descent (PGD) method \cite{madry2017towards}. Moreover, the Wasserstein distance ball constraint in \cref{eq:opt} is intractable for high-dimensional data. As a result, we relax the intractable constraint by considering a subset of $\mathcal{B}_{W_\infty(\mu_S, \epsilon)}$, denoted by
	  \begin{equation}\label{eq:cont}
	      \mathcal{B}(p(S), \epsilon) = \{(\Mat{A}_i,{\Mat{X}_{i}^{\prime}})\mid \|\Mat{X}_i - {\Mat{X}_i}^\prime\|_{2,\infty}<\epsilon, i\in \mathcal{V}\},
	  \end{equation}
	  where $\|\cdot\|_{2,\infty}$ denotes the $\ell_{2,\infty}$ norm of a matrix, \ie, the maximum $\ell_2$ norm of the rows of the matrix; $\mathcal{V}$ refers to the node set of the whole graph; $\Mat{A}_i$, $\Mat{X}_i$ are the adjacency matrix and feature matrix of subgraph corresponding to node $i$, respectively. 

    \subsection{Mutual Information Estimation}

    Mutual information is notoriously intractable to optimize for high-dimensional data since it involves integral on the unknown prior distributions. In this section, we derive tractable bound for the optimization of mutual information $I(S;Z^\prime)$, $I(S;Z)$ and $I(S^\prime;Z^\prime|S)$ in the \cref{eq:opt_complete}, respectively.

    \textbf{Maximization of $\mathbf{I(S;Z^\prime)}$ and $\mathbf{I(S;Z)}$.} Existing methods \cite{velickovic2019deep,peng2020graph} typically maximize mutual information on graphs in a contrastive fashion, assuming that the graphs are free from adversarial perturbations. Considering that the adversarial perturbations can hinder the distinguishing of positive and negative samples in the MI maximization, in this section, we propose a fine-grained mutual information estimator named \emph{Subgraph Mutual Information} (SMI) upon the adversarially perturbed graph. Our SMI estimates the mutual information with a subgraph-level summary, which can provide fine-grained information to help representations eliminate the influence of adversarial perturbations. Specifically, the mutual information $I(S;Z^\prime)$, $I(S;Z)$ can be estimated based on the Jensen-Shannon divergence (JSD) MI estimator \cite{nowozin2016f}, respectively
  \begin{equation*}
  \hat{I}(S;Z^\prime) = \mathbb{E}_{p_{S,Z^\prime}}\left[\log D(z^\prime, s)\right] +\mathbb{E}_{p_S,p_{Z^\prime}}\left[\log(1 - D(z^\prime,s))\right],
  \end{equation*}
  \begin{equation*}
\hat{I}(S;Z) = \mathbb{E}_{p_{S,Z}}\left[\log D(z, s)\right]+\mathbb{E}_{p_{S},p_{Z}}\left[\log(1 - D(z, s))\right],
  \end{equation*}
    The discriminator $D$ decides whether the representations and the subgraphs are correlated, which consists of a Simplifying Graph Convolutional Network (SGC) \cite{wu2019simplifying} and a bilinear network. The SGC is leveraged to extract the subgraph-level fine-grained summaries from the subgraphs. The bilinear network is adopted to discriminate whether the summary and the node representations are correlated. Specifically, given subgraph $S$, the discriminator $D$ can be formalized as  
 \begin{equation} \label{eq:dis}
 D(z,s) = \text{Bilinear}\left(\text{SGC}\left(s\right), z\right).
 \end{equation}

    \textbf{Minimization of $\mathbf{I(S^\prime;Z^\prime|S)}$.} To minimize $I(S^\prime;Z^\prime|S)$, we derive a tractable variational upper bound. According to the definition of conditional mutual information, we have the following.

    \begin{equation}
        \begin{aligned}
            I(S^\prime;Z^\prime|S) &= \mathbb{E}_{s,s^\prime,z^\prime\sim p(S,S^\prime,Z^\prime)}\left[\log \frac{p(s)p(s^\prime,z^\prime,s)}{p(s^\prime,s)p(z^\prime,s)}\right]\\
        \end{aligned}
    \end{equation}
    According to the Markov Chain assumption in \cref{definition}, we have $p(s,s^\prime,z^\prime)=p(z^\prime|s^\prime)p(s^\prime|s)p(s)$. Thus, the mutual information can be rewritten as:
     \begin{equation}
        \begin{aligned}
            I(S^\prime;Z^\prime|S) &= \mathbb{E}_{s,s^\prime\sim p(S,S^\prime)}\mathbb{E}_{z^\prime\sim p_\theta(Z^\prime|S^\prime)}\left[\log \frac{p_\theta(z^\prime|s^\prime)}{p(z^\prime|s)}\right]\\ 
            &= \mathbb{E}_{s,s^\prime\sim p(S,S^\prime)}\mathbb{E}_{z^\prime\sim p_\theta(Z^\prime|S^\prime)}\left[\log \frac{p_\theta(z^\prime|s^\prime)p_\theta(z|s)}{p_\theta(z|s)p(z^\prime|s)}\right]\\           &=D_{\text{KL}}\left(p_\theta\left(Z^\prime|S^\prime\right)||p_\theta\left(Z|S\right)\right) - \\ &\ \ \ \ \ D_{\text{KL}}\left(p_\theta\left(Z|S\right)||p\left(Z^\prime|S\right)\right)\\
            &\leq D_{\text{KL}}\left(p_\theta\left(Z^\prime|S^\prime\right)||p_\theta\left(Z|S\right)\right)
        \end{aligned}
    \end{equation}
      where $p_\theta(Z|S)$ and $p_\theta(Z^\prime|S^\prime)$ denote the encoder parameterized by the learnable parameters $\theta$.\footnote{In practice, we assume that $p_\theta(z|s)\sim \mathcal{N}(\mu_\theta(s),\sigma_\theta(s))$ where $\mu_\theta(s)$ and $\sigma_\theta(s)$ are neural networks with learnable parameters $\theta$.} The above equation demonstrates that $I(S^\prime;Z^\prime|S)$ is upper bounded by $D_{\text{KL}}\left(p_\theta\left(Z^\prime|S^\prime\right)||p_\theta\left(Z|S\right)\right)$. The upper bound indicates that we can discard the adversarial information in the adversarial graph by minimizing the KL divergence between the benign representations and adversarial representations. By minimizing the upper bound, we can effectively reduce the influence of adversarial perturbations on the graph, leading to more robust representations.

 \subsection{Complexity Analysis}   
    We theoretically compare the computational cost of our RGIB and state-of-the-art method GRV~\cite{xu2022unsupervised}.  The main difference in computational cost between the two methods lies in the calculation of adversarial perturbations. For our RGIB, which only considers feature perturbations, we only need to calculate the gradients of features. Consequently, the time cost can be $\mathcal{O}(n\_iter*nd)$ where $n\_iter$ is the number of iterations for PGD attack conducted on features. On the other hand, for GRV~\cite{xu2022unsupervised}, the graph PGD attack method \cite{xu2019topology} is adopted to perturb structure, \ie, all entries in the adjacency matrix should be considered. As a result, the time complexity can be $\mathcal{O}(T*n^2+n\_iter*nd)$ where $T$ is the number of iterations for the graph PGD attack method. 
    The complexity is summarized in \cref{tab:complexity}. Note that $m \ll n^2$, we assert that our RGIB, which only involves feature perturbations, is much more efficient than GRV, which consider both structure and feature perturbations. Moreover, we experimentally prove this point in the experiments section.

	 \begin{table}[t]
	\centering
	\normalsize
	\setlength{\tabcolsep}{6pt}
	\caption{Complexity of two methods with different training strategies}\label{tab:complexity}
	\resizebox{.3\textwidth}{!}{
		\begin{tabular}{ccc}
			\toprule
			 &time complexity\cr
			\midrule
			GRV & $\mathcal{O}(T*n^2+n\_{iter}*nd)$\cr
                RGIB(ours) & $\mathcal{O}(n\_{iter}*nd)$ \cr
			\bottomrule
		\end{tabular}	}
\end{table}

	\section{theoretical analysis}
	In this section, we theoretically prove that our proposed RGIB benefits the robustness of downstream classification regardless of the specific form of classifiers. To this end, we introduce two lemmas and a definition first.
	\begin{lemma}{(Fano's Inequality).}\label{lemma:fano-inequality}
		Let $Y$ be a random variable uniformly distributed over a finite set of outcomes $\mathcal{Y}$. For any estimator $\hat{Y}$ such that $Y - X - \hat{Y}$ forms a Markov chain, we have 
		\begin{equation*}
		\Pr(\hat{Y}\neq Y) \geq 1 -  \frac{I(Y, \hat{Y}) + \log 2}{\log |\mathcal{Y}|}.
		\end{equation*}
	\end{lemma} 
	\begin{lemma}{(Data-Processing Inequality).}\label{lemma:data-processing}
		For any Markov chain $X - Y - Z$, we have 
		\begin{equation*}
		I(X;Y) \geq I(X;Z) \ and \ I (Y;Z) \geq I(X;Z).
		\end{equation*}
	\end{lemma}
	
	\begin{definition}{(Adversarial Risk \cite{zhu2020learning}).}
			Let $(\mathcal{X}, \Delta)$ be the input metric space and $\mathcal{Y}$ be the set of labels. Let $\mu_{XY}$ be the underlying distribution of the input and label pairs. For any classifier $f : \mathcal{X} \rightarrow \mathcal{Y}$, the adversarial risk of $f$ with respect to $\epsilon \geq 0$ is defined as
			\begin{equation*}
			\!\!\!\!  AdvRisk_\epsilon(f) = \Pr_{x,y \sim p(X,Y)}\left[\exists x^\prime \in \mathcal{B}(x, \epsilon) \ s.t. f(x^\prime) \neq y\right].
			\end{equation*}
		\end{definition}
	

	Adversarial risk measures the vulnerability of a given classifier to adversarial perturbations. The lower the adversarial risk is, the more robust the classifier $f$ is. We prove a theoretical connection between our RGIB objective in \cref{eq:opt} and the adversarial risk of downstream classifiers. Specifically, when $\beta=0, \alpha=1$, our RGIB objective $\hat{\mathcal{L}}_{RGIB}$ works as a lower bound for the adversarial risk regardless of the specific form of downstream classifiers.
		
		\begin{theorem}\label{th2}
		Let $(\mathcal{S}, \Delta)$ be the input metric space, $\mathcal{Y}$ be the set of labels and $p(S,Y)$ be the underlying joint probability distribution. Assume the marginal distribution of labels $p(Y)$ is a uniform distribution over $\mathcal{Y}$. Consider the feature space $\mathcal{Z}$ and the set of downstream classifier $\mathcal{F}=\{f:\mathcal{Z}\rightarrow \mathcal{Y}\}$. Given $\epsilon\geq 0$, for any encoder $e:\mathcal{S}\rightarrow \mathcal{Z}$, the following inequality holds
		\begin{equation*}\label{eq:conclusion}
		\begin{aligned}
		\inf_{f\in \mathcal{F}}AdvRisk_\epsilon(f\circ e) &\geq 1 - \frac{\inf_{p(S^\prime)\in \mathcal{B}_{W_\infty}}I(S;Z^\prime)+\log 2}{\log |\mathcal{Y}|} \\
        &\geq 1 - \frac{\inf_{p(S^\prime)\in \mathcal{B}_{W_\infty}}I(S^\prime;Z^\prime)+\log 2}{\log |\mathcal{Y}|}. . 
		\end{aligned}
		\end{equation*}
	\end{theorem}

	\begin{proof}	
		\noindent For any given $p(S^\prime) \in \mathcal{B}_{W_\infty}$ and $f\in \mathcal{F}$, we have the following Markov chain
		\begin{equation} \label{eq:markov}
		Y - S - S^\prime  \overset{e}{-} Z^\prime \overset{f}{-} \hat{Y},
		\end{equation}
		\noindent where $\hat{Y} = f(e(S^\prime))$ indicates the downstream prediction made by $f$ on the perturbed graph $S^\prime$. By applying \cref{lemma:fano-inequality} and \cref{lemma:data-processing}, we arrive at the following inequality	\begin{equation}\label{eq:lemma_inequality}
		\begin{aligned}
		\Pr (\hat{Y}\neq Y) &\geq 1 - \frac{I(S;Z^\prime) + \log 2}{\log |\mathcal{Y}|}.\\ 
		\end{aligned}
		\end{equation}
		For any classifier $f:\mathcal{Z \rightarrow \mathcal{Y}}$, according to \cite{zhu2020learning}, the adversarial risk can be rewritten as:	
        \begin{equation} \label{eq:rewitten}
        \begin{aligned}
            AdvRisk_\epsilon(f\circ e) &= \sup_{p(S^\prime)\in \mathcal{B}_{W_\infty}} \Pr\left(f(e(S^\prime))\neq Y\right) \\
            &= \sup_{p(S^\prime)\in \mathcal{B}_{W_\infty}} \Pr(\hat{Y}\neq Y).
        \end{aligned}
		\end{equation}
        Substituting \cref{eq:lemma_inequality} into \cref{eq:rewitten}, we get a lower bound of adversarial risk		\begin{equation}\label{eq:lower_bound}
		\begin{aligned}
		&\ \ \ \ \ \inf_{f\in\mathcal{F}}\left[AdvRisk_\epsilon(f\circ e)\right] \\
		&\geq \inf_{f\in\mathcal{F}}\sup_{p(S^\prime)\in \mathcal{B}_{W_\infty}} \left[1 - \frac{I(S;Z^\prime) + \log 2}{\log |\mathcal{Y}|}\right]. \\
		\end{aligned}
		\end{equation}
		\noindent Note that the right side of \cref{eq:lower_bound} is irrelevant with $f$, thus the infimum over $f$ can be omitted. By taking the irrelevant item out of the supremum, we can derive the first inequality:	\begin{equation*}
		\begin{aligned}
		\inf_{f\in\mathcal{F}} AdvRisk_\epsilon(f\circ e) 
        \geq 1 -  \frac{\inf_{p(S^\prime)\in \mathcal{B}_{W_\infty}}I(S;Z^\prime) + \log 2}{\log |\mathcal{Y}|}\\
		\end{aligned}
		\end{equation*}
        Furthermore, according to the \cref{th:factor}, we can prove that 
        \begin{equation}
        \begin{aligned}
            I(S^\prime; Z^\prime) = I(S^\prime; Z^\prime|S) + I(S;Z^\prime) \geq I(S;Z^\prime)
        \end{aligned}
        \end{equation}
       The second inequality can be easily proved by taking into above inequality. The proof is completed. 
	\end{proof}

	From \cref{th2}, the first inequality proves that our RGIB objective $\hat{\mathcal{L}}_{RGIB}$ in \cref{eq:opt} provides a lower bound for the minimum of adversarial risk when $\beta=0, \alpha=1$. In other words, any downstream classifier cannot be robust if the RGIB objective is low. This connection indicates that we can get robust node representations without the knowledge of downstream classifiers and label information by maximizing the infimum of RGIB instead. The right side of the second inequality is the formulation of the Infomax objective in GRV~\cite{xu2022unsupervised}, which involves a biased assumption. In the second inequality in \cref{th2},  we theoretically prove that our unbiased RGIB provides a tighter lower bound for the minimum of adversarial risk comparing with the Infomax objective. Experimental results in the experiments section confirm the theoretical analysis and demonstrate the superiority of our method.

\begin{table*}[]
	\centering
	\caption{Summary of results for the node classification, link prediction, and community detection tasks on both benign graphs and adversarial graphs. Bold figures indicate the best performance and underlined figures indicate the second best method.} \label{tab:acc}
	\vspace{-0.05in}
 \setlength{\tabcolsep}{2pt}
\resizebox{\textwidth}{!}{
	\begin{tabular}{l|l|cccc|cccc|cccc} \toprule
		& & \multicolumn{4}{c|}{Node classification (ACC\%)} & \multicolumn{4}{c|}{Link prediction (AUC\%)} & \multicolumn{4}{c}{Community detection (NMI\%)} \\ 
		& Model &Cora &Citeseer &Cora\_ML &Pubmed &Cora &Citeseer  &Cora\_ML &Pubmed &Cora &Citeseer  &Cora\_ML &Pubmed \\ \hline
\multirow{8}{*}{\rotatebox{90}{Benign}}&
GAE&$67.5\pm 2.3$&$47.9\pm 17.2$&$51.3\pm 5.3$&$71.7\pm 3.0$&$76.1\pm 1.1$&$76.3\pm 6.4$&$77.6\pm 1.7$&$\underline{86.3\pm 0.6}$&$37.8\pm 2.7$&$25.8\pm 5.0$&$23.6\pm 3.1$&$19.8\pm 5.0$\cr
&DGI&$\underline{79.5\pm 0.8}$&$69.7\pm 0.7$&$77.4\pm 0.9$&$76.0\pm 0.8$&$80.5\pm 1.1$&$80.6\pm 1.0$&$86.0\pm 1.0$&$79.3\pm 0.5$&$\bf{54.6\pm 0.9}$&$42.6\pm 0.6$&$54.3\pm 1.0$&$25.7\pm 0.4$\cr
&GRACE&$75.9\pm 2.0$&$66.9\pm 1.0$&$74.2\pm 3.4$&$\bf{78.7\pm 0.7}$&$\bf{87.6\pm 0.9}$&$\bf{91.4\pm 1.2}$&$84.2\pm 1.3$&$86.1\pm 1.4$&$50.0\pm 1.4$&$36.7\pm 2.1$&$49.2\pm 3.2$&$23.4\pm 3.6$\cr
&MVGRL&$\bf{79.8\pm 0.8}$&$68.4\pm 1.1$&$72.0\pm 2.0$&$72.0\pm 1.9$&$81.6\pm 0.9$&$78.4\pm 1.8$&$84.5\pm 0.7$&$73.7\pm 1.2$&$55.2\pm 0.7$&$37.5\pm 1.1$&$55.8\pm 1.1$&$\underline{26.3\pm 2.7}$\cr
&DGI-Jaccard&$79.4\pm 0.6$&$\bf{70.1\pm 0.7}$&$77.4\pm 0.8$&$76.6\pm 0.6$&$79.6\pm 1.2$&$80.8\pm 1.1$&$\underline{86.1\pm 1.1}$&$79.3\pm 0.6$&$\underline{54.4\pm 0.9}$&$42.6\pm 0.5$&$54.2\pm 0.9$&$25.7\pm 0.4$\cr
&DGI-SVD&$74.2\pm 0.7$&$69.5\pm 0.7$&$62.6\pm 1.3$&OOM&$75.8\pm 0.8$&$78.7\pm 1.2$&$79.9\pm 1.8$&OOM&$48.2\pm 0.5$&$40.4\pm 0.6$&$47.3\pm 0.7$&OOM\cr
&GRV&$76.4\pm 0.6$&$\underline{70.0\pm 0.9}$&$\underline{78.6\pm 1.0}$&$74.5\pm 1.3$&$82.3\pm 2.3$&$86.4\pm 0.9$&$86.0\pm 1.5$&$77.2\pm 1.0$&$54.2\pm 1.1$&$\underline{43.2\pm 0.8}$&$\bf{58.3\pm 0.6}$&$22.4\pm 2.4$\cr
&RGIB&$78.5\pm 0.5$&$68.8\pm 1.1$&$\bf{80.3\pm 0.8}$&$\underline{77.8\pm 0.7}$&$\underline{87.5\pm 0.6}$&$\underline{89.9\pm 1.7}$&$\bf{88.2\pm 1.6}$&$\bf{87.8\pm 0.3}$&$53.2\pm 1.7$&$\bf{44.0\pm 0.6}$&$\underline{58.0\pm 0.9}$&$\bf{26.3\pm 0.6}$\cr
\midrule
\multirow{8}{*}{\rotatebox{90}{Adv}}&
GAE&$49.3\pm 2.4$&$13.7\pm 3.0$&$36.5\pm 3.8$&$19.0\pm 2.9$&$76.2\pm 1.4$&$75.4\pm 4.1$&$75.2\pm 1.6$&$\underline{86.0\pm 0.6}$&$31.9\pm 2.4$&$16.6\pm 2.5$&$12.2\pm 1.1$&$16.7\pm 2.0$\cr
&DGI&$45.1\pm 1.2$&$9.2\pm 0.7$&$29.3\pm 1.0$&$12.8\pm 0.6$&$75.5\pm 1.0$&$76.3\pm 1.7$&$73.5\pm 1.5$&$77.9\pm 0.5$&$39.4\pm 3.6$&$29.3\pm 0.6$&$27.4\pm 1.4$&$\underline{25.0\pm 0.4}$\cr
&GRACE&$46.6\pm 3.3$&$12.3\pm 2.0$&$\underline{36.9\pm 3.4}$&$10.7\pm 1.2$&$\underline{85.1\pm 0.7}$&$\underline{83.4\pm 1.7}$&$76.8\pm 1.8$&$84.6\pm 1.3$&$42.2\pm 1.8$&$25.8\pm 3.0$&$23.7\pm 1.8$&$24.0\pm 3.4$\cr
&MVGRL&$57.1\pm 1.7$&$11.9\pm 0.8$&$28.0\pm 0.2$&$10.1\pm 2.1$&$74.5\pm 1.2$&$73.2\pm 0.7$&$66.6\pm 0.8$&$70.9\pm 1.2$&$37.7\pm 1.5$&$19.8\pm 0.7$&$7.2\pm 2.7$&$23.1\pm 7.0$\cr
&DGI-Jaccard&$52.5\pm 1.7$&$9.3\pm 0.7$&$29.5\pm 1.1$&$12.9\pm 0.4$&$73.5\pm 1.1$&$76.1\pm 1.8$&$73.4\pm 1.5$&$77.9\pm 0.6$&$39.5\pm 3.3$&$29.6\pm 0.5$&$26.6\pm 1.8$&$25.0\pm 0.4$\cr
&DGI-SVD&$66.6\pm 2.6$&$23.4\pm 1.5$&$29.0\pm 0.3$&OOM&$58.0\pm 0.6$&$71.5\pm 0.4$&$58.1\pm 1.8$&OOM&$35.3\pm 1.7$&$23.6\pm 0.3$&$\underline{31.7\pm 3.1}$&OOM\cr
&GRV&$\underline{71.5\pm 2.2}$&$\underline{26.4\pm 5.7}$&$34.7\pm 1.5$&$\underline{28.5\pm 1.8}$&$82.4\pm 2.1$&$82.4\pm 4.5$&$\underline{78.9\pm 1.7}$&$79.9\pm 0.7$&$\underline{49.6\pm 2.8}$&$\underline{31.4\pm 3.0}$&$30.5\pm 4.7$&$24.0\pm 2.9$\cr
&RGIB&$\bf{74.9\pm 1.0}$&$\bf{38.0\pm 6.6}$&$\bf{47.7\pm 4.1}$&$\bf{60.4\pm 1.7}$&$\bf{87.2\pm 1.0}$&$\bf{89.3\pm 1.8}$&$\bf{85.6\pm 2.3}$&$\bf{87.4\pm 0.5}$&$\bf{49.9\pm 1.3}$&$\bf{41.3\pm 2.4}$&$\bf{47.8\pm 2.0}$&$\bf{26.7\pm 0.5}$\cr

\bottomrule
\end{tabular}
 }
\vspace{-0.05in}
\end{table*}

\begin{figure*}[htb]
	\centering
    \includegraphics[width=\textwidth]{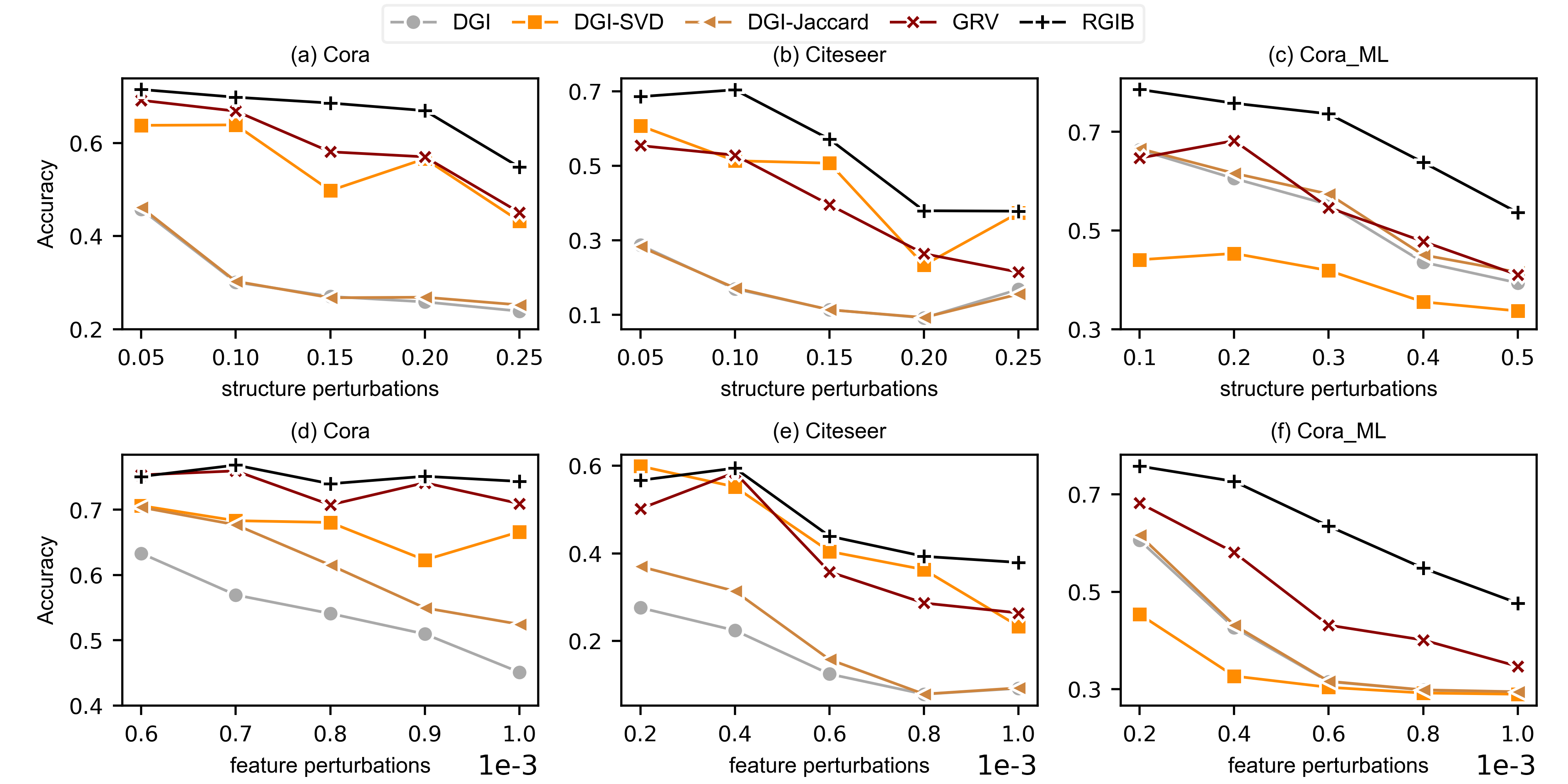}
    \caption{Classification accuracy against adversarial attack with different budgets.}\label{fig:acc_diff_perturbations}
\end{figure*}

\section{experiments}
In this section, we conduct extensive experimental studies to evaluate our proposed RGIB method. We first introduce the details of our adopted datasets, baseline methods, and experimental settings. And then, we discuss the results of the experiments. Specifically, we aim to answer the following questions: 
\begin{itemize}
	    \item \textbf{RQ1} How does our proposed method perform compared to other baseline methods when the graph is adversarially attacked?
	    \item \textbf{RQ2} Do our proposed adversarial training strategy and mutual information estimator benefit the performance of our method as expected?
	    \item \textbf{RQ3} How do different hyper-parameters affect the performance of our method?
\end{itemize}

\subsection{Experimental Setup}
\subsubsection{Datasets} 
We conduct experiments on four benchmark datasets, including Cora, Citeseer, Pubmed \cite{sen2008collective} and Cora\_ML \cite{mccallum2000automating}. The four datasets are collected from citation networks where the nodes correspond to documents and edges indicate citation relationships. The statistical details are shown in \cref{tab:datasets}. For Cora, Citeseer, and Pubmed, we follow the standard data split \cite{kipf2016semi}. For Cora\_ML, the nodes are randomly split to train, validation, and test sets with 10\%, 10\% and 80\% of the nodes, separately.

\begin{table}[t]
	\centering
	\normalsize
	\setlength{\tabcolsep}{2pt}
	\caption{Dataset statistics}\label{tab:datasets}
	\resizebox{.4\textwidth}{!}{
		\begin{tabular}{lrrrrr}
			\toprule
			Statics&\#Nodes&\#Edges&\#Features&\#Classes&  \makecell[c]{Average\\Degree}\cr
			\midrule
			Cora &$2,708$&$5,278$&$1,433$&$7$&$4.90$\cr
			Citeseer&$3,327$&$4,552$&$3,703$&$6$&$3.74$\cr
			Cora\_ML&$2,995$&$8,158$&$2,879$&$7$&$6.45$\cr
			Pubmed&$19,717$&$44,324$&$500$&$3$&$5.50$\cr
			\bottomrule
		\end{tabular}	}
		\vspace{-0.5cm}
\end{table}

\subsubsection{Baselines} 
As far as we know, there are no existing methods focusing on unsupervised robust representation learning except for GRV\cite{xu2022unsupervised}. Thus we consider four non-robust graph representation learning methods (GAE, DGI, MVGRL, and GRACE), two preprocessing-based defense methods (DGI-Jaccard, DGI-SVD) and GRV, \ie, a robust representation learning method. These methods are briefly described as follows:
	\begin{itemize}
	\item \textbf{GAE \cite{kipf2016variational}.} GAE is a widely used unsupervised method for graph data, which learns the representations of nodes with a GCN based auto-encoder.
	\item \textbf{DGI \cite{velickovic2019deep}.} Deep Graph Infomax (DGI) is a method that learns node representations by maximizing mutual information between local and global representations of a graph.
	\item \textbf{MVGRL \cite{hassani2020contrastive}.} MVGRL is a self-supervised approach for node-level and graph-level representation learning that contrasts multiple structural views of graphs to learn representations.
	\item \textbf{GRACE \cite{Zhu2020vf}} GRACE is a method that employs graph augmentation at both topology and attribute levels to construct diverse node contexts for contrastive learning.
	\item \textbf{DGI-Jaccard \cite{wu2019adversarial}.}DGI-Jaccard is a preprocessing-based defense method that removes edges with low Jaccard similarity before applying DGI to learn the representations on the preprocessed graph.
	\item \textbf{DGI-SVD \cite{entezari2020all}.}  DGI-SVD is another preprocessing-based defense method that denoises the adjacency matrix using a low-rank approximation before applying DGI to learn the representations on the preprocessed graph.
	\item \textbf{GRV \cite{xu2022unsupervised}.} Graph Representation Vulnerability (GRV) is a robustness measure that aims to learn robust node representations by maximizing the mutual information (MI) between the adversarial graph and the corresponding adversarial representation.
    \end{itemize}

\subsubsection{Implementation Details} In our experiments, we adopt a two-layer GCN \cite{kipf2016semi} as the encoder for our proposed RGIB method. The hidden dimensions of GCN are set as 512 on all datasets. For simplicity, the SGC in \cref{eq:dis} shares the same parameters with GCN.
In the preprocessing procedure, the features of nodes are normalized with $\ell_2$ norm. In the training phase, we adopt the PGD method~\cite{madry2017towards} to optimize the feature perturbations and the Adam optimizer to optimize the encoder with learning rate as $0.001$. In the evaluation phase, we employ the graph PGD attack \cite{xu2019topology} and the PGD method to generate global perturbations on the structure and node features, separately. We consider three different downstream tasks to evaluate the robustness of node representations: node classification, link prediction, and community detection. For node classification, logistic regression is adopted as the downstream classifier and the classification accuracy is reported. For link prediction, we adopt the logistic regression to predict whether an edge exists or not. $10\%$ edges are sampled as positive test set and the same number of edges are randomly sampled as the negative test set. The remaining edges is used as the training set. We report the area under the curve (AUC) as the evaluation metric for link prediction. For community detection, we detect the potential communities with k-means algorithm and report the normalized mutual information (NMI) to measure the performance. All experiments are conducted on a machine with CPU E5-2650 v4 @ 2.20GHz and a 2080 Ti GPU with 12GB RAM.

\begin{table*}[]
	\centering
	\caption{Summary of results for the node classification against various adversarial attack methods. Bold figures indicate the best performance and underlined figures indicate the second best method.} \label{tab:acc_diff_attack}
	\vspace{-0.05in}
\resizebox{\textwidth}{!}{
	\begin{tabular}{l|ccc|ccc|ccc} \toprule
		& \multicolumn{3}{c|}{MinMax} & \multicolumn{3}{c|}{Metattack} & \multicolumn{3}{c}{CLGA} \\ 
		 Model &Cora &Citeseer &Cora\_ML &Cora &Citeseer  &Cora\_ML &Cora &Citeseer  &Cora\_ML \\ \hline
GAE&$58.1\pm 2.9$&$57.3\pm 1.6$&$57.7\pm 3.6$&$38.3\pm 17.9$&$29.6\pm 8.9$&$21.5\pm 3.8$&$\underline{43.0\pm 4.0}$&$37.8\pm 3.6$&$38.2\pm 2.5$\cr
DGI&$71.2\pm 1.0$&$62.3\pm 0.8$&$65.0\pm 0.9$&$58.5\pm 0.9$&$37.8\pm 1.3$&$18.1\pm 1.4$&$32.0\pm 1.9$&$30.0\pm 0.7$&$32.8\pm 0.9$\cr
GRACE&$65.9\pm 2.1$&$60.5\pm 1.6$&$61.1\pm 3.6$&$58.7\pm 3.0$&$40.1\pm 1.7$&$20.0\pm 1.5$&$42.2\pm 1.8$&$\underline{42.5\pm 1.2}$&$\underline{40.4\pm 1.8}$\cr
MVGRL&$\underline{72.1\pm 1.1}$&$66.1\pm 0.6$&$68.4\pm 0.8$&$55.5\pm 3.0$&$43.0\pm 1.5$&$21.0\pm 2.0$&$29.9\pm 0.5$&$28.9\pm 0.1$&$28.4\pm 0.1$\cr
DGI-Jaccard&$72.1\pm 1.2$&$\underline{66.6\pm 1.3}$&$\underline{70.5\pm 1.1}$&$58.5\pm 1.1$&$38.7\pm 1.3$&$17.7\pm 1.3$&$32.1\pm 1.9$&$29.8\pm 0.8$&$32.9\pm 0.8$\cr
DGI-SVD&$69.0\pm 1.5$&$65.4\pm 0.9$&$68.6\pm 1.0$&$60.0\pm 1.7$&$48.4\pm 1.5$&$36.8\pm 1.0$&$30.9\pm 0.4$&$28.5\pm 0.1$&$28.3\pm 0.1$\cr
GRV&$72.0\pm 1.7$&$64.8\pm 1.2$&$67.8\pm 1.0$&$\underline{61.5\pm 7.1}$&$\underline{54.7\pm 6.5}$&$\underline{47.3\pm 8.8}$&$40.6\pm 1.1$&$35.7\pm 0.9$&$34.9\pm 1.0$\cr
RGIB&$\bf{74.4\pm 0.8}$&$\bf{67.7\pm 0.4}$&$\bf{71.2\pm 0.9}$&$\bf{66.1\pm 3.0}$&$\bf{58.3\pm 1.6}$&$\bf{60.7\pm 2.0}$&$\bf{54.6\pm 4.9}$&$\bf{43.6\pm 3.1}$&$\bf{47.4\pm 4.6}$\cr
\bottomrule
\end{tabular}
 }
\vspace{-0.05in}
\end{table*}

\begin{table*}[t]
	\normalsize
	\centering
	\caption{Visualization of representations with t-SNE on Cora} \label{tab:rep}
	\setlength{\tabcolsep}{0pt}
	\resizebox{0.9\textwidth}{!}{
		\begin{tabular}{ccc}
			  \scriptsize{DGI} & \scriptsize{GRV} & \scriptsize{RGIB}\\  
			\begin{minipage}[b]{0.2\textwidth}
				\centering				\raisebox{-.5\height}{\includegraphics[width=0.8\linewidth]{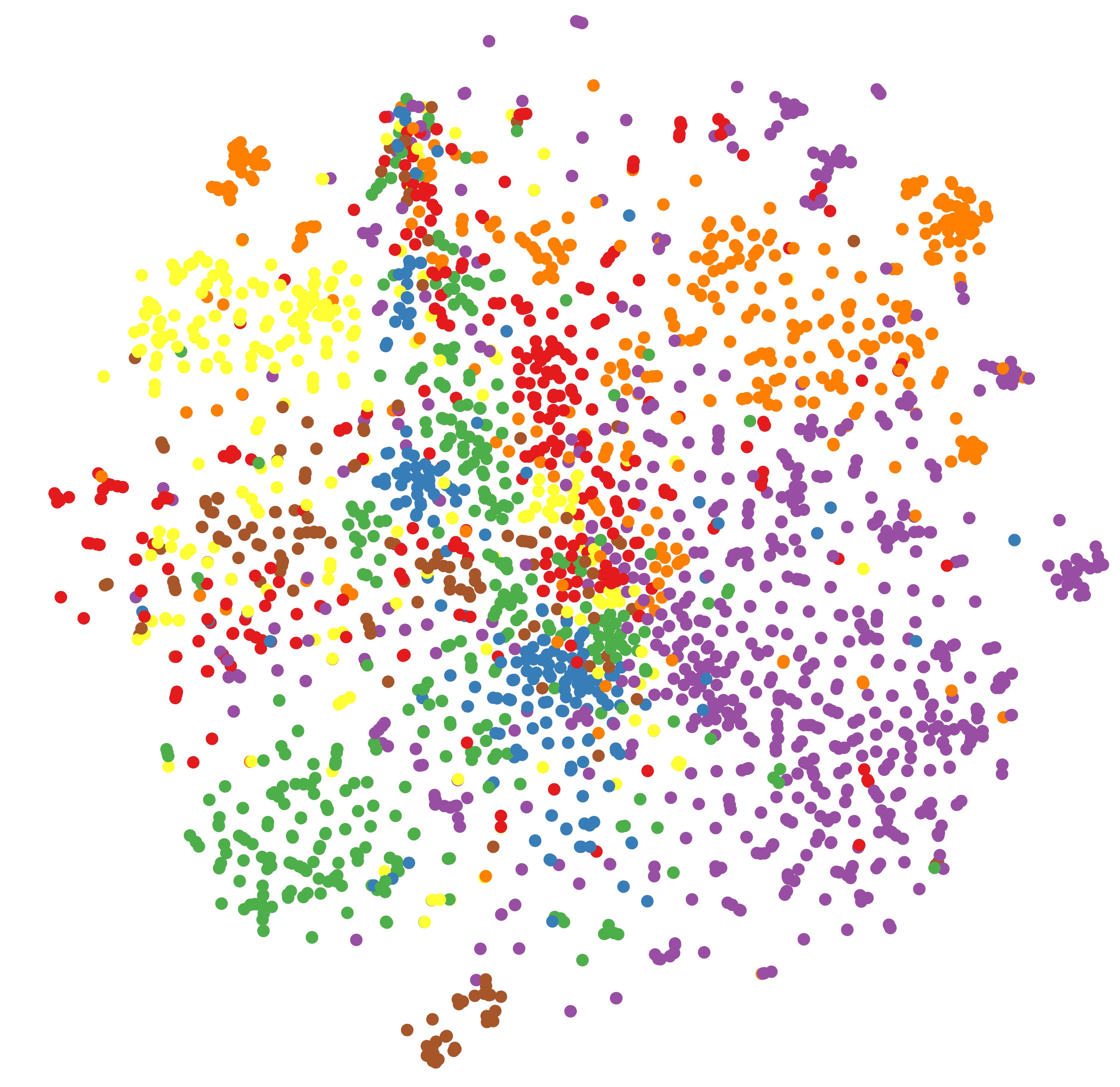}}
			\end{minipage}
			&\begin{minipage}[b]{0.2\textwidth}
				\centering				\raisebox{-.5\height}{\includegraphics[width=0.8\linewidth]{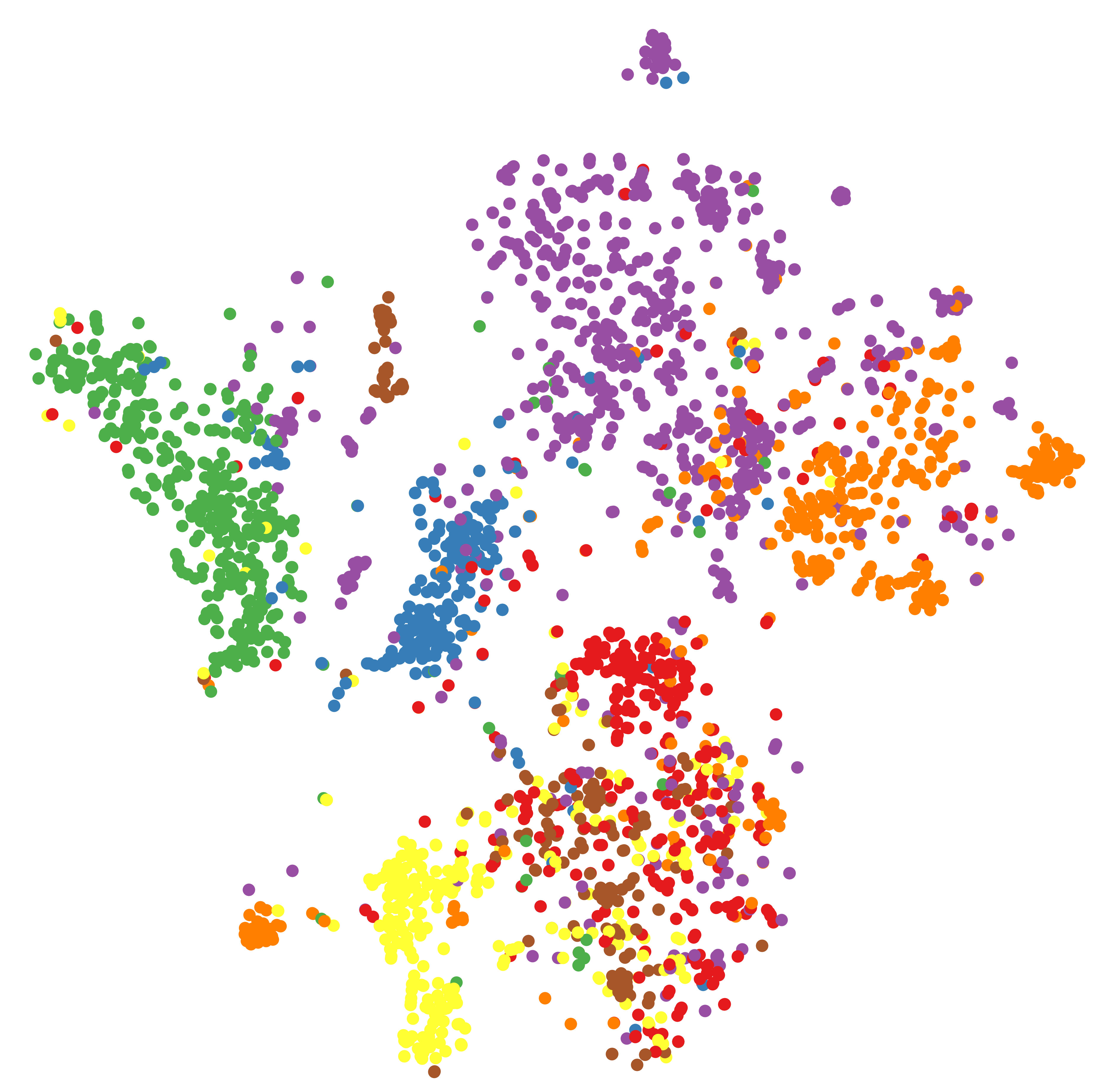}}
			\end{minipage}
			& \begin{minipage}[b]{0.2\textwidth}
				\centering				\raisebox{-.5\height}{\includegraphics[width=0.8\linewidth]{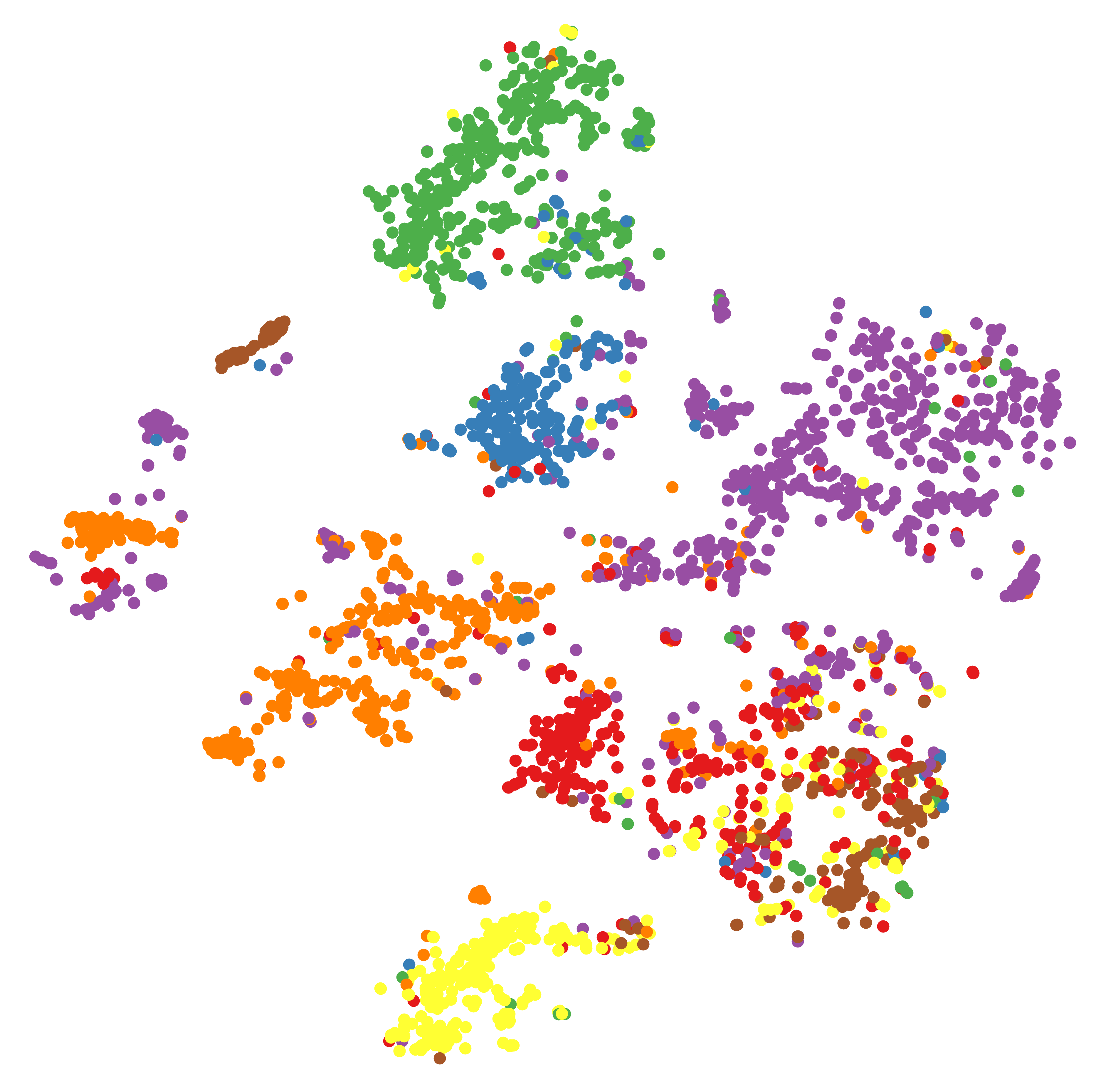}}
			\end{minipage}\cr
	\end{tabular}}
\end{table*}

\subsection{Robustness Against Adversarial Attack}

To answer \textbf{RQ1}, we conduct extensive experiments to investigate the performance of our method. 

\textbf{Performance on various downstream tasks.} In \cref{tab:acc}, we summarize the performance of baselines and our method on various tasks. The bold highlighted numbers indicate the best among all methods and the underlined numbers indicate the second best method. For the PGD attack on features, the budget is fixed as $0.001$, and for the graph PGD attack on structure, we set the budget as $0.2*|\mathcal{E}|$ where $|\mathcal{E}|$ denotes the number of edges in the graph. From \cref{tab:acc}, we can find the following observations: (i) our proposed RGIB outperforms other baseline methods on the adversarial graphs and achieve competitive performance on the benign graphs for various tasks. Specifically, for the Cora dataset without perturbations, the state-of-the-art method MVGRL achieves $79.8\%$ classification accuracy, while our method is comparable with $78.5\%$. When the dataset is adversarially perturbed, the classification performance of MVGRL drops significantly to $57.1\%$, while our method remains robust with a $74.9\%$ accuracy. (ii) Compared to the non-robust methods (\ie, GAE, DGI, MVGRL and GRACE) and the preprocessing-based methods (\ie, DGI-Jaccard and DGI-SVD), robust node representations learning methods, \ie, GRV and RGIB, perform better on various tasks and datasets. This observation demonstrates the superiority of robust learning-based  methods. (iii) Our proposed RGIB outperforms GRV on the adversarial graphs. We attribute this observation to the fact that RGIB does not rely on any biased assumptions. This observation demonstrates that the tighter lower bound of adversarial risk (\ie, \textit{AdvRisk}) provided by RGIB contributes to the robustness of node representations.

\textbf{Performance under different adversarial budgets.} To further investigate the performance of unsupervised graph representation learning methods against adversarial attacks with varying budgets, we present the node classification accuracy with different levels of structural perturbations and feature perturbations in \cref{fig:acc_diff_perturbations}. From the figure, we can observe consistent results with our previous findings, where our RGIB demonstrates better robustness compared to other baseline methods as the perturbation budget varies. This observation highlights the effectiveness of our approach in maintaining robust node representations even when faced with different degrees of adversarial perturbations.

\textbf{Performance against other adversarial attack methods.}  
In \cref{tab:acc_diff_attack}, we investigate the performance of our method against different adversarial attack methods. Specifically, we adopt three widely-used attack methods: 1) MinMax~\cite{xu2019topology} generates adversarial perturbations by iteratively optimizing a min-max optimization problem; 2) Metattack~\cite{zugner_adversarial_2019} employs meta-gradients to solve the bilevel adversarial attack problem underlying training-time attacks, essentially treating the graph as a hyperparameter to optimize; 3) CLGA~\cite{zhang2022unsupervised} is a novel unsupervised gradient-based adversarial attack method that does not rely on labels for graph contrastive learning. Since these attack methods cannot be directly applied to our model, we conduct these attack strategies in a black-box setting. Specifically, a surrogate GCN model is trained in a supervised fashion, and then perturbations are generated based on the surrogate model using the three attack methods mentioned above. We report the classification accuracy on the adversarial graphs generated by the three methods in \cref{tab:acc_diff_attack}. Results show that our RGIB outperforms other baseline methods across various attack methods.

\textbf{Qualitative visualization.} For an intuitive illustration, we visualize the adversarial representations learned from Cora using t-SNE \cite{van2008visualizing} in \cref{tab:rep}. From a qualitative perspective, we can find that the intercluster boundary of DGI gets blurred after adversarial attacks, whereas GRV and RGIB stay more robust. Specifically, for our proposed RGIB method, the intercluster boundaries are clearer to other methods. For example, the nodes labeled as orange and purple have wider decision boundaries for RGIB than DGI and GRV. This observation indicates that our method improves the robustness of representations, resulting in more stable intercluster boundaries even when the graphs are subjected to attacks.

\begin{table}[t]
	\centering
	\normalsize
	\setlength{\tabcolsep}{5pt}
	\caption{Settings of ablation studies.}\label{tab:ab_settings}
	\resizebox{0.4\textwidth}{!}{
		\begin{tabular}{ccc}
		\toprule
            &\makecell[c]{MI estimator}&\makecell[c]{Training strategy }\cr
            \hline
            RGIB-S&SMI&structure and feature\cr
            RGIB-DGI&DGI&feature-only\cr
            RGIB&SMI&feature-only \cr
		\bottomrule
		\end{tabular}
	}
\end{table}

\begin{figure}[ht]
    \centering
    \includegraphics[width=0.45\textwidth]{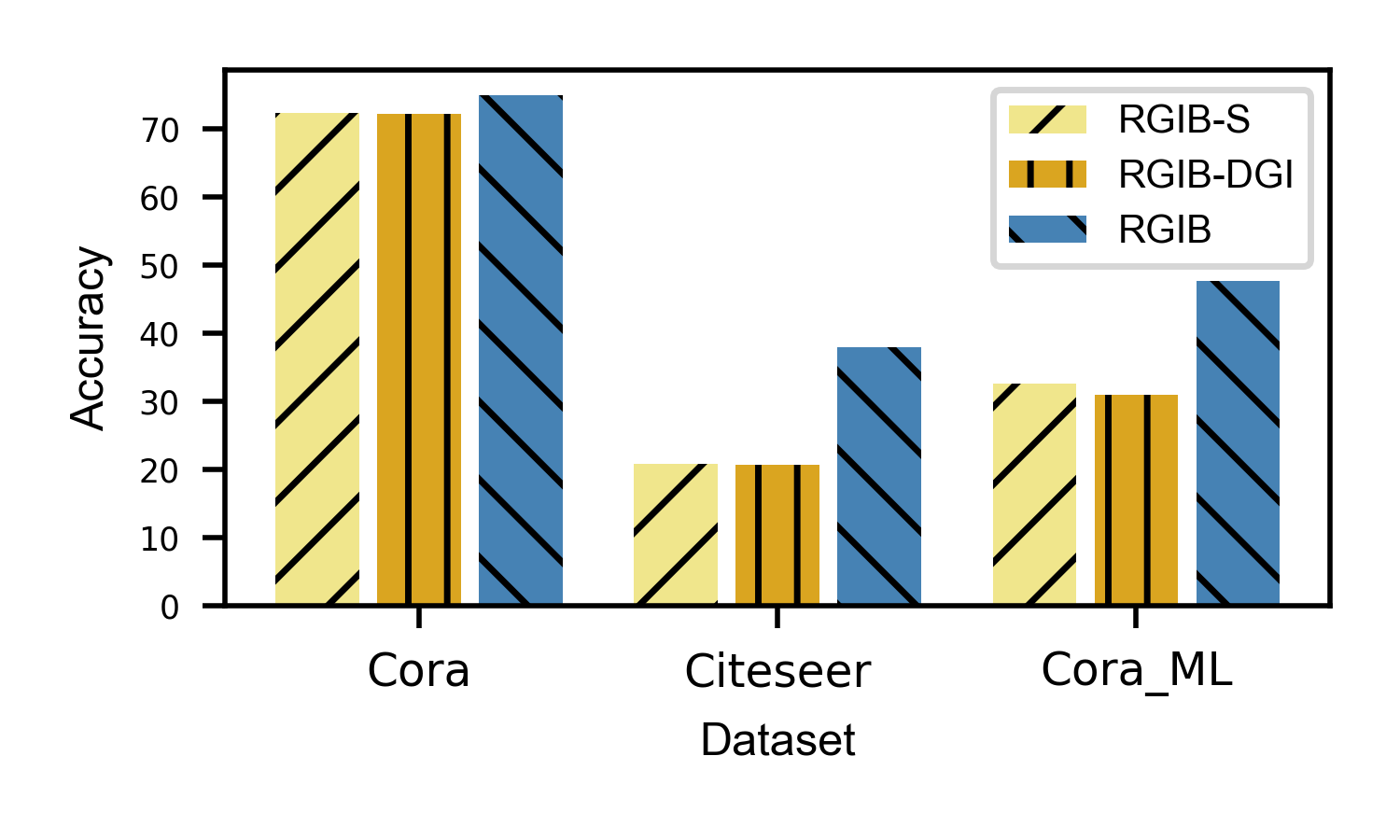}
    \caption{Ablation sturdy on the mutual information estimator and the training strategy.}
    \label{fig:ablation}
\end{figure}
\vspace{-0.2cm}
\subsection{Ablation studies}

To optimize our RGIB objective, we employ an effective mutual information estimator with a fine-grained summary and an efficient training strategy that only considers feature perturbations. To answer \textbf{RQ2}, we conduct ablation studies to show how these points contribute to the performance of our method in \cref{fig:ablation}. In \cref{fig:ablation}, we denote RGIB-S as a variant of RGIB which consider both the structure perturbations and feature perturbations during training. And RGIB-DGI is a variant of RGIB which estimates the mutual information using DGI~\cite{velickovic2019deep} instead of our SMI. We summarize the settings of different variants in \cref{tab:ab_settings}. From the figure, we make the following observations: (i) Our RGIB outperforms RGIB-DGI.  We attribute the improvement in performance to the fine-grained summary provided by SMI. Specifically, DGI extracts the graph-level information by a readout function (mean pooling on all node representations), which is coarse-grained and leads to information loss. Our proposed SMI leverage subgraph-level summary (extracted by SGC) to provide fine-grained information when fed into the discriminator along with node representations. The fine-grained information can help to identify perturbations on node representations more precisely and make the encoder more robust against perturbations. (ii) Our RGIB outperforms RGIB-S. A reasonable explanation for this observation is that the structure perturbations involved in the RGIB-S increase the difficulty of convergence since the structure perturbations are discrete. In contrast, RGIB only considers continuous feature perturbations, making it easier to converge.

Moreover, to investigate how the feature-only perturbation strategy affects the efficiency of our method, we report the running time in \cref{tab:time_training_procedure}. We find that RGIB, which only perturbs features in the training procedure, is much more efficient than methods like RGIB-S and GRV, which jointly perturb the structure and features. For example, an epoch will cost 0.49s on Cora if only perturbing features and 2.69s if structure and features are perturbed simultaneously, \ie, it will lead to $5.5\times$ time cost increase if the structure is perturbed in the training procedure. Furthermore, the joint perturbation training strategy cannot scale to larger graphs like Pubmed because of the high memory costs.

\begin{table}[t]
	\centering
	\normalsize
	\setlength{\tabcolsep}{2pt}
	\caption{Time cost of different training strategies. N/A denotes the method can not scale to the graph due to limitation of GPU RAM.}\label{tab:time_training_procedure}
	\resizebox{0.35\textwidth}{!}{
		\begin{tabular}{lcccccc}
			\toprule
			\multirow{2}{*}{Method}&
			\multicolumn{4}{c}{Time}\cr
			\cmidrule(lr){2-5}
			&Cora&Citeseer&Cora\_ML&Pubmed\cr
			\midrule
			GRV &3.19s&4.41s&3.73s&N/A\cr
			RGIB-S &2.69s&4.17s&3.52s&N/A\cr
			RGIB &0.49s&0.88s&0.64s&3.30s\cr
			\bottomrule
		\end{tabular}
	}
\end{table}

\begin{figure}[t]
	\setlength{\abovecaptionskip}{-0.3cm}
    \centering
    \includegraphics[width=0.4\textwidth]{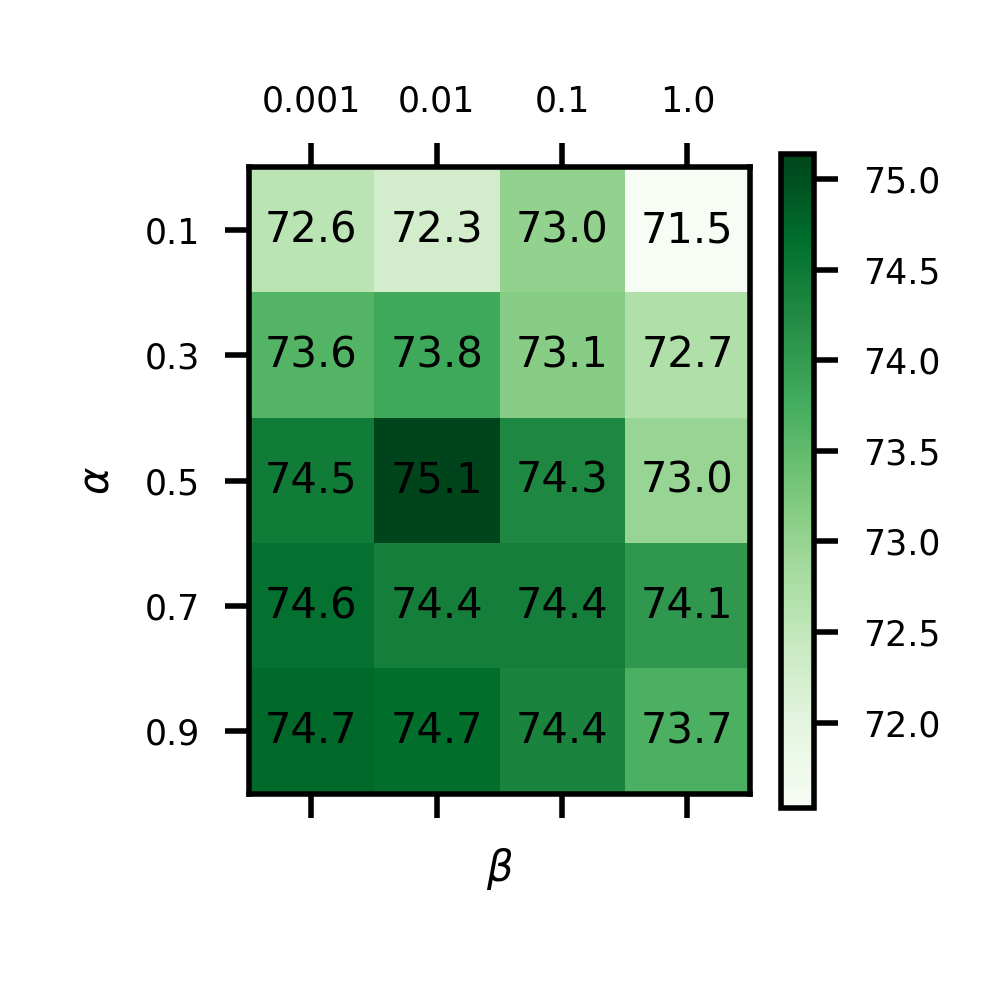}
    \caption{Parameter sturdy on the hyper-parameter.}
    \label{fig:parameter}
\end{figure}

\subsection{Parameter Study}
To answer \textbf{RQ3}, we conduct parameter studies to explore the effect of different values of hyper-parameters, \ie, $\alpha$ and $\beta$ in the \cref{eq:opt_complete}. Specifically, we report the classification accuracy in \cref{fig:parameter} with different values of the hyper-parameters.  In \cref{fig:parameter}, we can observe that either too small value or too large value of the hyper-parameters will lead to sub-optimal performance. It is essential to strike a proper balance between the effectiveness and robustness, as well as between the informative term and the adversarial term.

\section{CONCLUSION}
In this paper, we focus on robust unsupervised graph representation learning. To avoid the biased assumption in Informax objective, we extend the Information Bottleneck (IB) principle to learn robust node representations, resulting in a novel method Robust Graph Information Bottleneck (RGIB). Our RGIB simultaneously preserves the original information in the benign graph and eliminates the adversarial information in the adversarial graph. In comparison to existing methods, our proposed RGIB establishes a tighter connection with downstream classifiers and does not rely on any biased assumptions. Furthermore, to optimize the graph representation robustness effectively and efficiently, we propose a mutual information estimator with fine-grained graph summary and a training strategy that only involves feature perturbation. Experimental results on various benchmark datasets and downstream tasks demonstrate the effectiveness and efficiency of our proposed method in improving the robustness of UGRL against adversarial attacks.

\section*{Acknowledgment}
Jihong Wang, Minnan Luo, and Qinghua Zheng are supported by the National Key Research and Development Program of China (No. 2022YFB3102600), National Nature Science Foundation of China (No. 62192781, No. 62272374, No. 62202367, No. 62250009, No. 62137002, No. 61937001), Innovative Research Group of the National Natural Science Foundation of China (61721002), Innovation Research Team of Ministry of Education (IRT\_17R86), Project of China Knowledge Center for Engineering Science and Technology, and Project of Chinese academy of engineering ``The Online and Offline Mixed Educational Service System for `The Belt and Road' Training in MOOC China''. They would like to express our gratitude for the support of K. C. Wong Education Foundation. The authors also would like to thank the reviewers and chairs
for their constructive feedback and suggestions.

	\ifCLASSOPTIONcompsoc

	\ifCLASSOPTIONcaptionsoff
	\newpage
	\fi

	\bibliographystyle{IEEEtran}
	\bibliography{reference}

\begin{thebibliography}{10}
\providecommand{\url}[1]{#1}
\csname url@samestyle\endcsname
\providecommand{\newblock}{\relax}
\providecommand{\bibinfo}[2]{#2}
\providecommand{\BIBentrySTDinterwordspacing}{\spaceskip=0pt\relax}
\providecommand{\BIBentryALTinterwordstretchfactor}{4}
\providecommand{\BIBentryALTinterwordspacing}{\spaceskip=\fontdimen2\font plus
\BIBentryALTinterwordstretchfactor\fontdimen3\font minus
  \fontdimen4\font\relax}
\providecommand{\BIBforeignlanguage}[2]{{%
\expandafter\ifx\csname l@#1\endcsname\relax
\typeout{** WARNING: IEEEtran.bst: No hyphenation pattern has been}%
\typeout{** loaded for the language `#1'. Using the pattern for}%
\typeout{** the default language instead.}%
\else
\language=\csname l@#1\endcsname
\fi
#2}}
\providecommand{\BIBdecl}{\relax}
\BIBdecl

\bibitem{velickovic2019deep}
P.~Velickovic, W.~Fedus, W.~L. Hamilton, P.~Li{\`o}, Y.~Bengio, and R.~D.
  Hjelm, ``Deep graph infomax.'' in \emph{ICLR (Poster)}, 2019.

\bibitem{peng2020graph}
Z.~Peng, W.~Huang, M.~Luo, Q.~Zheng, Y.~Rong, T.~Xu, and J.~Huang, ``Graph
  representation learning via graphical mutual information maximization,'' in
  \emph{Proceedings of The Web Conference 2020}, 2020, pp. 259--270.

\bibitem{8392745}
P.~Cui, X.~Wang, J.~Pei, and W.~Zhu, ``A survey on network embedding,''
  \emph{IEEE Transactions on Knowledge and Data Engineering}, vol.~31, no.~05,
  pp. 833--852, may 2019.

\bibitem{8519335}
R.~A. Rossi, R.~Zhou, and N.~K. Ahmed, ``Deep inductive graph representation
  learning,'' \emph{IEEE Transactions on Knowledge and Data Engineering},
  vol.~32, no.~03, pp. 438--452, mar 2020.

\bibitem{8941296}
H.~Wang, J.~Wang, J.~Wang, M.~Zhao, W.~Zhang, F.~Zhang, W.~Li, X.~Xie, and
  M.~Guo, ``Learning graph representation with generative adversarial nets,''
  \emph{IEEE Transactions on Knowledge and Data Engineering}, vol.~33, no.~08,
  pp. 3090--3103, aug 2021.

\bibitem{ye2020symmetrical}
S.~Ye, J.~Liang, R.~Liu, and X.~Zhu, ``Symmetrical graph neural network for
  quantum chemistry with dual real and momenta space,'' \emph{The Journal of
  Physical Chemistry A}, vol. 124, no.~34, pp. 6945--6953, 2020.

\bibitem{fung2021benchmarking}
V.~Fung, J.~Zhang, E.~Juarez, and B.~G. Sumpter, ``Benchmarking graph neural
  networks for materials chemistry,'' \emph{npj Computational Materials},
  vol.~7, no.~1, pp. 1--8, 2021.

\bibitem{8326519}
L.~Liao, X.~He, H.~Zhang, and T.~Chua, ``Attributed social network embedding,''
  \emph{IEEE Transactions on Knowledge and Data Engineering}, vol.~30, no.~12,
  pp. 2257--2270, dec 2018.

\bibitem{9001178}
F.~Guo, Y.~Yuan, G.~Wang, L.~Chen, X.~Lian, and Z.~Wang, ``Cohesive group
  nearest neighbor queries on road-social networks under multi-criteria,''
  \emph{IEEE Transactions on Knowledge and Data Engineering}, vol.~33, no.~11,
  pp. 3520--3536, nov 2021.

\bibitem{yuan2020xgnn}
H.~Yuan, J.~Tang, X.~Hu, and S.~Ji, ``Xgnn: Towards model-level explanations of
  graph neural networks,'' in \emph{Proceedings of the 26th ACM SIGKDD
  International Conference on Knowledge Discovery \& Data Mining}, 2020, pp.
  430--438.

\bibitem{kipf2016semi}
T.~N. Kipf and M.~Welling, ``Semi-supervised classification with graph
  convolutional networks,'' in \emph{International Conference on Learning
  Representations (ICLR)}, 2017.

\bibitem{hamilton2017inductive}
W.~L. Hamilton, R.~Ying, and J.~Leskovec, ``Inductive representation learning
  on large graphs,'' in \emph{NIPS}, 2017.

\bibitem{velivckovicgraph}
P.~Veli{\v{c}}kovi{\'c}, G.~Cucurull, A.~Casanova, A.~Romero, P.~Li{\`o}, and
  Y.~Bengio, ``Graph attention networks,'' in \emph{International Conference on
  Learning Representations}.

\bibitem{8901123}
Q.~Zhang, T.~Chu, and C.~Zhang, ``Semi-supervised graph based embedding with
  non-convex sparse coding techniques,'' \emph{IEEE Transactions on Knowledge
  and Data Engineering}, vol.~33, no.~05, pp. 2193--2207, may 2021.

\bibitem{feng2019graph}
F.~Feng, X.~He, J.~Tang, and T.-S. Chua, ``Graph adversarial training:
  Dynamically regularizing based on graph structure,'' \emph{IEEE Transactions
  on Knowledge and Data Engineering}, 2019.

\bibitem{zugner2018adversarial}
D.~Z{\"u}gner, A.~Akbarnejad, and S.~G{\"u}nnemann, ``Adversarial attacks on
  neural networks for graph data,'' in \emph{Proceedings of the 24th ACM SIGKDD
  International Conference on Knowledge Discovery \& Data Mining}, 2018, pp.
  2847--2856.

\bibitem{zugner2019adversarial}
D.~Z{\"u}gner and S.~G{\"u}nnemann, ``Adversarial attacks on graph neural
  networks via meta learning,'' in \emph{International Conference on Learning
  Representations (ICLR)}, 2019.

\bibitem{xu2019topology}
K.~Xu, H.~Chen, S.~Liu, P.-Y. Chen, T.-W. Weng, M.~Hong, and X.~Lin, ``Topology
  attack and defense for graph neural networks: An optimization perspective,''
  in \emph{International Joint Conference on Artificial Intelligence (IJCAI)},
  2019.

\bibitem{fan2020graph}
W.~Fan, Y.~Ma, Q.~Li, J.~Wang, G.~Cai, J.~Tang, and D.~Yin, ``A graph neural
  network framework for social recommendations,'' \emph{IEEE Transactions on
  Knowledge and Data Engineering}, 2020.

\bibitem{ying2018graph}
R.~Ying, R.~He, K.~Chen, P.~Eksombatchai, W.~L. Hamilton, and J.~Leskovec,
  ``Graph convolutional neural networks for web-scale recommender systems,'' in
  \emph{Proceedings of the 24th ACM SIGKDD International Conference on
  Knowledge Discovery \& Data Mining}, 2018, pp. 974--983.

\bibitem{fan2019graph}
W.~Fan, Y.~Ma, Q.~Li, Y.~He, E.~Zhao, J.~Tang, and D.~Yin, ``Graph neural
  networks for social recommendation,'' in \emph{The World Wide Web
  Conference}, 2019, pp. 417--426.

\bibitem{jin2019power}
M.~Jin, H.~Chang, W.~Zhu, and S.~Sojoudi, ``Power up! robust graph
  convolutional network against evasion attacks based on graph powering,'' in
  \emph{Proceedings of the Thirty-Fifth Conference on Association for the
  Advancement of Artificial Intelligence (AAAI)}, 2021, pp. 8004--8012.

\bibitem{zhang2020gnnguard}
X.~Zhang and M.~Zitnik, ``Gnnguard: Defending graph neural networks against
  adversarial attacks,'' in \emph{NeurIPS}, 2020.

\bibitem{xu2022unsupervised}
J.~Xu, Y.~Yang, J.~Chen, X.~Jiang, C.~Wang, J.~Lu, and Y.~Sun, ``Unsupervised
  adversarially robust representation learning on graphs,'' in
  \emph{Proceedings of the AAAI Conference on Artificial Intelligence},
  vol.~36, no.~4, 2022, pp. 4290--4298.

\bibitem{tishby2000information}
N.~Tishby, F.~C. Pereira, and W.~Bialek, ``The information bottleneck method,''
  \emph{arXiv preprint physics/0004057}, 2000.

\bibitem{alemi2016deep}
A.~A. Alemi, I.~Fischer, J.~V. Dillon, and K.~Murphy, ``Deep variational
  information bottleneck,'' \emph{arXiv preprint arXiv:1612.00410}, 2016.

\bibitem{wu2019simplifying}
F.~Wu, A.~Souza, T.~Zhang, C.~Fifty, T.~Yu, and K.~Weinberger, ``Simplifying
  graph convolutional networks,'' in \emph{International conference on machine
  learning}.\hskip 1em plus 0.5em minus 0.4em\relax PMLR, 2019, pp. 6861--6871.

\bibitem{perozzi2014deepwalk}
B.~Perozzi, R.~Al-Rfou, and S.~Skiena, ``Deepwalk: Online learning of social
  representations,'' in \emph{Proceedings of the 20th ACM SIGKDD international
  conference on Knowledge discovery and data mining}, 2014, pp. 701--710.

\bibitem{tang2015line}
J.~Tang, M.~Qu, M.~Wang, M.~Zhang, J.~Yan, and Q.~Mei, ``Line: Large-scale
  information network embedding,'' in \emph{Proceedings of the 24th
  international conference on world wide web}, 2015, pp. 1067--1077.

\bibitem{qiu2018network}
J.~Qiu, Y.~Dong, H.~Ma, J.~Li, K.~Wang, and J.~Tang, ``Network embedding as
  matrix factorization: Unifying deepwalk, line, pte, and node2vec,'' in
  \emph{Proceedings of the eleventh ACM international conference on web search
  and data mining}, 2018, pp. 459--467.

\bibitem{yang2008non}
J.~Yang, S.~Yang, Y.~Fu, X.~Li, and T.~Huang, ``Non-negative graph embedding,''
  in \emph{2008 IEEE Conference on Computer Vision and Pattern
  Recognition}.\hskip 1em plus 0.5em minus 0.4em\relax IEEE, 2008, pp. 1--8.

\bibitem{kipf2016variational}
T.~N. Kipf and M.~Welling, ``Variational graph auto-encoders,'' \emph{arXiv
  preprint arXiv:1611.07308}, 2016.

\bibitem{garcia2017learning}
A.~Garcia~Duran and M.~Niepert, ``Learning graph representations with embedding
  propagation,'' \emph{Advances in neural information processing systems},
  vol.~30, pp. 5119--5130, 2017.

\bibitem{pan2018adversarially}
S.~Pan, R.~Hu, G.~Long, J.~Jiang, L.~Yao, and C.~Zhang, ``Adversarially
  regularized graph autoencoder for graph embedding.'' in \emph{IJCAI}, 2018,
  pp. 2609--2615.

\bibitem{tian2020contrastive}
Y.~Tian, D.~Krishnan, and P.~Isola, ``Contrastive multiview coding,'' in
  \emph{Computer Vision--ECCV 2020: 16th European Conference, Glasgow, UK,
  August 23--28, 2020, Proceedings, Part XI 16}.\hskip 1em plus 0.5em minus
  0.4em\relax Springer, 2020, pp. 776--794.

\bibitem{li2019graph}
Y.~Li, C.~Gu, T.~Dullien, O.~Vinyals, and P.~Kohli, ``Graph matching networks
  for learning the similarity of graph structured objects,'' in
  \emph{International conference on machine learning}.\hskip 1em plus 0.5em
  minus 0.4em\relax PMLR, 2019, pp. 3835--3845.

\bibitem{sun2019infograph}
F.-Y. Sun, J.~Hoffman, V.~Verma, and J.~Tang, ``Infograph: Unsupervised and
  semi-supervised graph-level representation learning via mutual information
  maximization,'' in \emph{International Conference on Learning
  Representations}, 2019.

\bibitem{zhu2021graph}
Y.~Zhu, Y.~Xu, F.~Yu, Q.~Liu, S.~Wu, and L.~Wang, ``Graph contrastive learning
  with adaptive augmentation,'' in \emph{Proceedings of the Web Conference
  2021}, 2021, pp. 2069--2080.

\bibitem{hassani2020contrastive}
K.~Hassani and A.~H. Khasahmadi, ``Contrastive multi-view representation
  learning on graphs,'' in \emph{International Conference on Machine
  Learning}.\hskip 1em plus 0.5em minus 0.4em\relax PMLR, 2020, pp. 4116--4126.

\bibitem{dai2018adversariala}
H.~Dai, H.~Li, T.~Tian, X.~Huang, L.~Wang, J.~Zhu, and L.~Song, ``Adversarial
  attack on graph structured data,'' in \emph{International conference on
  machine learning}.\hskip 1em plus 0.5em minus 0.4em\relax PMLR, 2018, pp.
  1115--1124.

\bibitem{bojchevski2019adversarial}
A.~Bojchevski and S.~G{\"u}nnemann, ``Adversarial attacks on node embeddings
  via graph poisoning,'' in \emph{International Conference on Machine
  Learning}.\hskip 1em plus 0.5em minus 0.4em\relax PMLR, 2019, pp. 695--704.

\bibitem{wang2019adversarial}
S.~Wang, Z.~Chen, J.~Ni, X.~Yu, Z.~Li, H.~Chen, and P.~S. Yu, ``Adversarial
  defense framework for graph neural network,'' \emph{arXiv preprint
  arXiv:1905.03679}, 2019.

\bibitem{jin2020graph}
W.~Jin, Y.~Ma, X.~Liu, X.~Tang, S.~Wang, and J.~Tang, ``Graph structure
  learning for robust graph neural networks,'' in \emph{Proceedings of the 26th
  ACM SIGKDD International Conference on Knowledge Discovery \& Data Mining},
  2020, pp. 66--74.

\bibitem{wu2019adversarial}
W.~Jin, Y.~Li, H.~Xu, Y.~Wang, S.~Ji, C.~Aggarwal, and J.~Tang, ``Adversarial
  attacks and defenses on graphs,'' \emph{SIGKDD Explor. Newsl.}, p. 19–34,
  2021.

\bibitem{entezari2020all}
N.~Entezari, S.~A. Al-Sayouri, A.~Darvishzadeh, and E.~E. Papalexakis, ``All
  you need is low (rank) defending against adversarial attacks on graphs,'' in
  \emph{Proceedings of the 13th International Conference on Web Search and Data
  Mining}, 2020, pp. 169--177.

\bibitem{federici2020}
M.~Federici, A.~Dutta, P.~Forré, N.~Kushman, and Z.~Akata, ``Learning robust
  representations via multi-view information bottleneck,'' in
  \emph{International Conference on Learning Representations}, 2020.

\bibitem{zhang2019theoretically}
H.~Zhang, Y.~Yu, J.~Jiao, E.~Xing, L.~El~Ghaoui, and M.~Jordan, ``Theoretically
  principled trade-off between robustness and accuracy,'' in
  \emph{International Conference on Machine Learning}.\hskip 1em plus 0.5em
  minus 0.4em\relax PMLR, 2019, pp. 7472--7482.

\bibitem{champion2008wasserstein}
T.~Champion, L.~De~Pascale, and P.~Juutinen, ``The infinity-wasserstein
  distance: Local solutions and existence of optimal transport maps,''
  \emph{SIAM Journal on Mathematical Analysis}, vol.~40, no.~1, pp. 1--20,
  2008.

\bibitem{sen2008collective}
P.~Sen, G.~Namata, M.~Bilgic, L.~Getoor, B.~Galligher, and T.~Eliassi-Rad,
  ``Collective classification in network data,'' \emph{AI magazine}, vol.~29,
  no.~3, pp. 93--93, 2008.

\bibitem{lecun2002efficient}
Y.~LeCun, L.~Bottou, G.~B. Orr, and K.-R. M{\"u}ller, ``Efficient backprop,''
  in \emph{Neural networks: Tricks of the trade}.\hskip 1em plus 0.5em minus
  0.4em\relax Springer, 2002, pp. 9--50.

\bibitem{madry2017towards}
A.~Madry, A.~Makelov, L.~Schmidt, D.~Tsipras, and A.~Vladu, ``Towards deep
  learning models resistant to adversarial attacks,'' in \emph{International
  Conference on Learning Representations (ICLR)}, 2018.

\bibitem{nowozin2016f}
S.~Nowozin, B.~Cseke, and R.~Tomioka, ``f-gan: Training generative neural
  samplers using variational divergence minimization,'' in \emph{Proceedings of
  the 30th International Conference on Neural Information Processing Systems},
  2016, pp. 271--279.

\bibitem{zhu2020learning}
S.~Zhu, X.~Zhang, and D.~Evans, ``Learning adversarially robust representations
  via worst-case mutual information maximization,'' in \emph{International
  Conference on Machine Learning}.\hskip 1em plus 0.5em minus 0.4em\relax PMLR,
  2020, pp. 11\,609--11\,618.

\bibitem{mccallum2000automating}
A.~K. McCallum, K.~Nigam, J.~Rennie, and K.~Seymore, ``Automating the
  construction of internet portals with machine learning,'' \emph{Information
  Retrieval}, vol.~3, no.~2, pp. 127--163, 2000.

\bibitem{Zhu2020vf}
\BIBentryALTinterwordspacing
Y.~Zhu, Y.~Xu, F.~Yu, Q.~Liu, S.~Wu, and L.~Wang, ``{Deep Graph Contrastive
  Representation Learning},'' in \emph{ICML Workshop on Graph Representation
  Learning and Beyond}, 2020. [Online]. Available:
  \url{http://arxiv.org/abs/2006.04131}
\BIBentrySTDinterwordspacing

\bibitem{zugner_adversarial_2019}
D.~Z{\"u}gner and S.~G{\"u}nnemann, ``Adversarial attacks on graph neural
  networks via meta learning,'' in \emph{International Conference on Learning
  Representations (ICLR)}, 2019.

\bibitem{zhang2022unsupervised}
S.~Zhang, H.~Chen, X.~Sun, Y.~Li, and G.~Xu, ``Unsupervised graph poisoning
  attack via contrastive loss back-propagation,'' in \emph{Proceedings of the
  ACM Web Conference 2022}, 2022, pp. 1322--1330.

\bibitem{van2008visualizing}
L.~Van~der Maaten and G.~Hinton, ``Visualizing data using t-sne.''
  \emph{Journal of machine learning research}, vol.~9, no.~11, 2008.

\end{thebibliography}
	%
	%
	%
	
	%

	
	\begin{IEEEbiography}[{\includegraphics[width=1in,height=1.25in,clip,keepaspectratio]{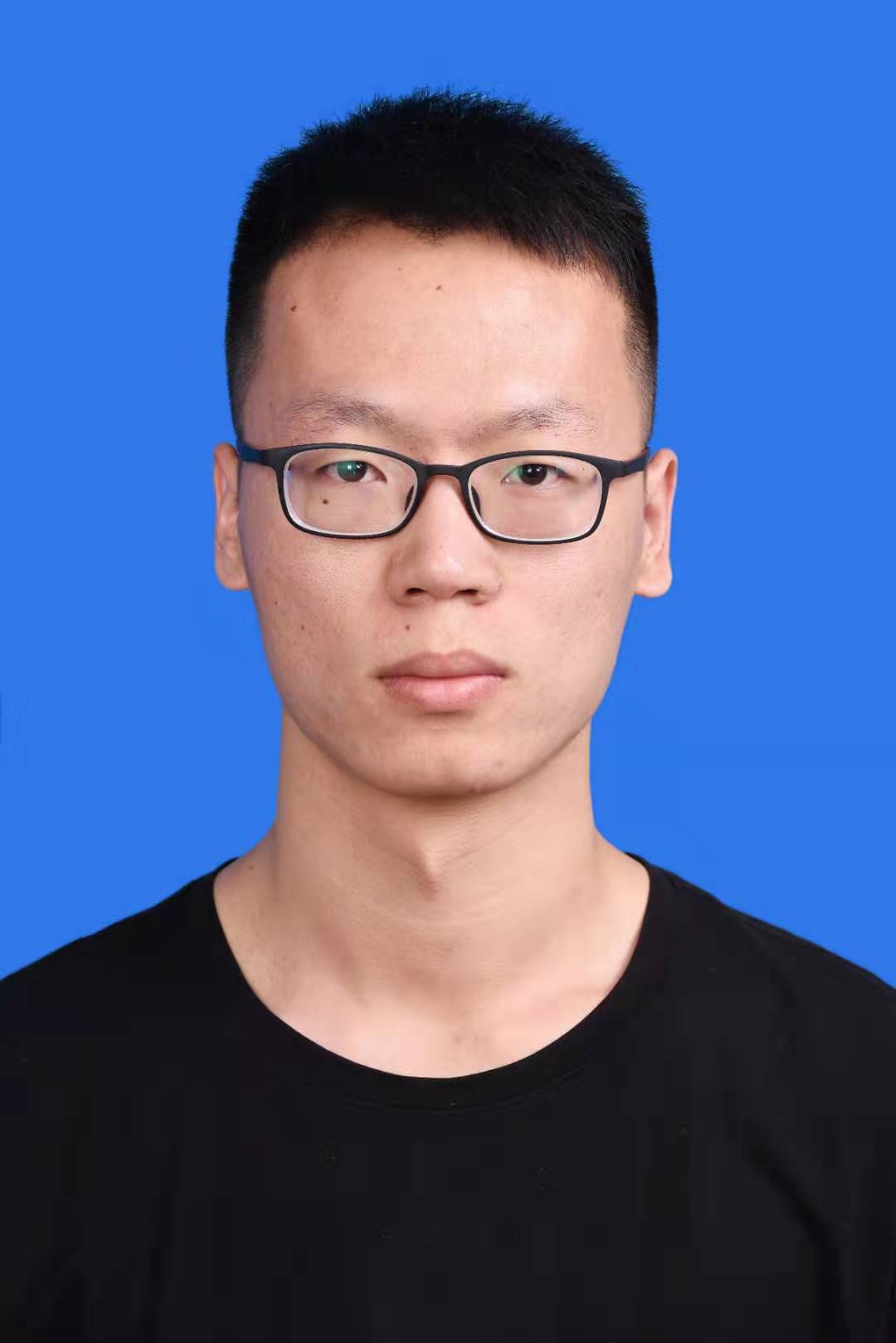}}]{Jihong Wang}
		received the B.Eng. degree from the School of Computer Science and Technology, Xi'an Jiaotong University, China, in 2019. He is currently a Ph.D. student in the School of Computer Science and Technology, Xi'an Jiaotong University, China. His research interests include robust machine learning and its applications, such as social computing and learning algorithms on graphs.
	\end{IEEEbiography}
	\vspace{-1cm}
	\begin{IEEEbiography}[{\includegraphics[width=1in,height=1.25in,clip,keepaspectratio]{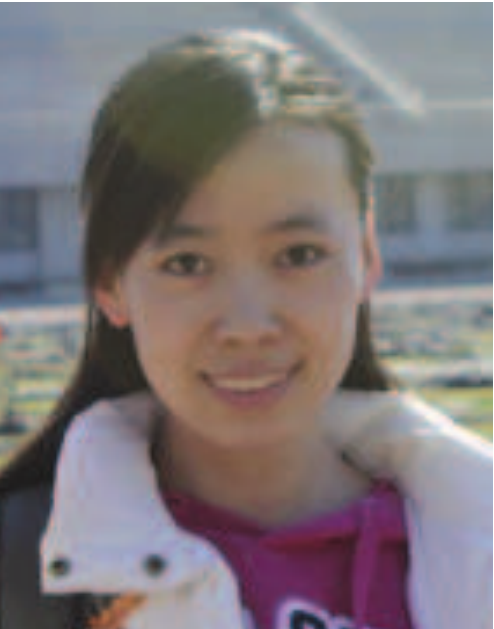}}]{Minnan Luo}
		received the Ph.D. degree from the Department of Computer Science and Technology, Tsinghua University, China, in 2014. She is currently an Associate Professor in the School of Electronic and Information Engineering at Xi'an Jiaotong University. Her research interests include machine learning and optimization, data mining, image processing, and cross-media retrieval.
	\end{IEEEbiography}
	\vspace{-1cm}
	\begin{IEEEbiography}[{\includegraphics[width=1in,height=1.25in,clip,keepaspectratio]{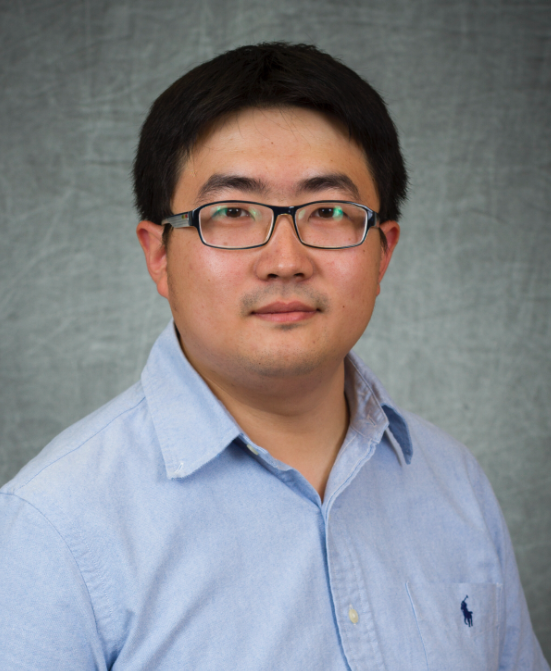}}]{Jundong Li}
		received the Ph.D. degree in Computer Science at Arizona State University in 2019. He is an Assistant Professor at the University of Virginia with appointments in Department of Electrical and Computer Engineering, Department of Computer Science. His research interests include data mining, machine learning and graph learning.
	\end{IEEEbiography}
 	\vspace{-1cm}

	\begin{IEEEbiography}[{\includegraphics[width=1in,height=1.25in,clip,keepaspectratio]{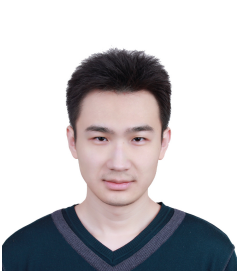}}]{Ziqi Liu}
		received the Ph.D. degree from the Department of Computer Science, Xi'an Jiaotong Univerity, China, in 2017. He was a visiting researcher advised by Prof. Alex Smola at Machine Learning Department, Carnegie Mellon University. He currently works at Ant Group. His research interests include  probabilistic models, graphical models, nonparametric modeling, large-scale inference algorithms and applications in user modeling, text mining.
		\end{IEEEbiography}
  	\vspace{-1cm}

    \begin{IEEEbiography}[{\includegraphics[width=1in,height=1.25in,clip,keepaspectratio]{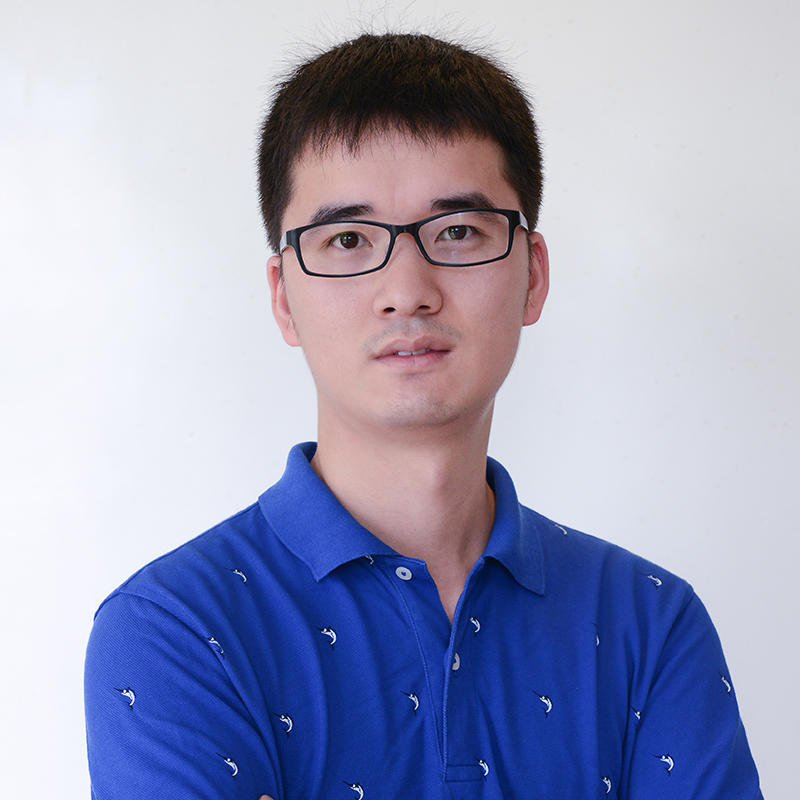}}]{Jun Zhou} is currently a Senior Staff Engineer at Ant Financial. His research mainly focuses on machine learning and data mining. He is a member of IEEE. He has published more than 40 papers in top-tier machine learning and data mining conferences, including VLDB, WWW, SIGIR, NeurIPS, AAAI, IJCAI, and KDD.
		\end{IEEEbiography}
		\vspace{-1cm}

	\begin{IEEEbiography}[{\includegraphics[width=1in,height=1.25in,clip,keepaspectratio]{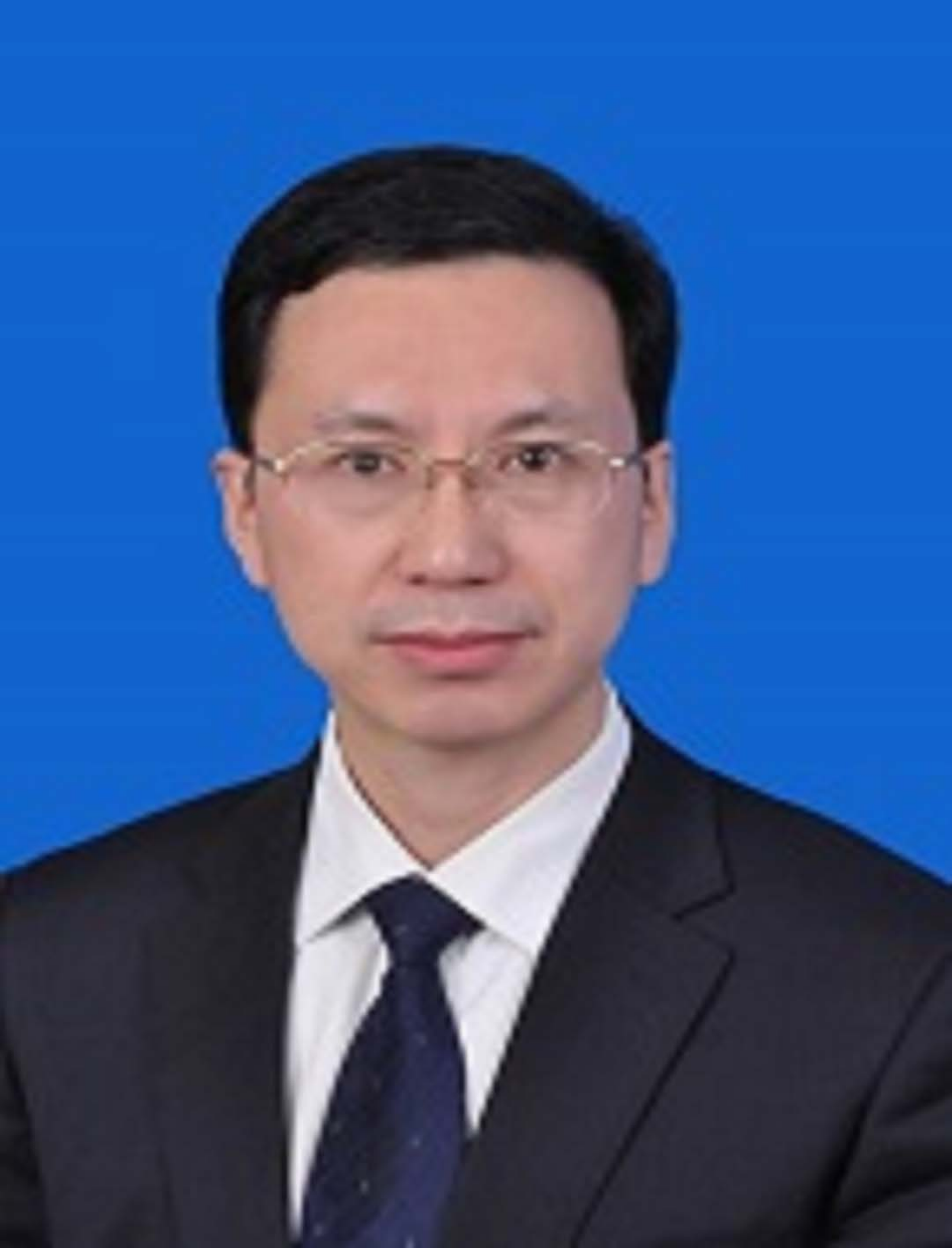}}]{Qinghua Zheng}
		received the Ph.D. degree of system engineering in 1997, the M.Sc. degree of computer organization and architecture in 1993, and the B.Eng. degree of computer software in 1990 at Xi’an Jiaotong University, China. He is currently a Professor in the School of Electronic and Information Engineering at Xi'an Jiaotong University. His research interests include computer network security, intelligent E-learning theory and algorithm, and multimedia e-learning.

	\end{IEEEbiography}
	
	
	

\end{document}